%% file: main.arxiv.tex
\documentclass[twocolumn]{article}

\usepackage{preprint}

\usepackage{amsmath,amssymb,amsfonts}
\usepackage{url}
\usepackage[numbers,sort&compress]{natbib}
\usepackage{hyperref}
\usepackage{cleveref}
\usepackage{color}
\usepackage{xspace}
\usepackage{mathtools}
\usepackage{graphicx}      
\usepackage{caption}
\usepackage{algorithm}
\usepackage{algorithmic}
\usepackage{wrapfig}
\usepackage{tabularx}
\usepackage{booktabs}
\usepackage{makecell}
\usepackage[inline]{enumitem}
\usepackage{tikz}
\usepackage{amsthm}

\input{preamble-AIDA}

\newtheorem{dfn}{Definition}
\newtheorem{theorem}{Theorem}
\newtheorem{lemma}{Lemma}
\newtheorem{remark}{Remark}

\newtheorem{proposition}{Proposition}
\newtheorem{corollary}{Corollary}

\newtheorem{generalboundednessassumption}{General boundedness assumption}
\crefname{generalboundednessassumption}{general boundedness
assumption}{general boundedness assumptions}
\Crefname{generalboundednessassumption}{General boundedness
assumption}{General boundedness assumptions}

\newcommand{\goaldist}[1][{}]{d_{\G#1}}
\newcommand{\learningpolicy}[1][{}]{
  \ifx#1\empty\policy^{\theta}
  \else\policy^{\theta_{#1}}\fi
}
\newcommand{\baselinepolicy}{\policy^{\mathrm{b}}}
\newcommand{\relprob}{p^{\mathrm{rel}}}

\newcommand{\kappalow}{\kappa^{\mathrm{low}}}
\newcommand{\kappahigh}{\kappa^{\mathrm{high}}}
\newcommand{\DT}{D^{T}}                              
\newcommand{\subscript}[2]{$#1 _ #2$}
\newcommand{\alglinelabel}{%
  \addtocounter{ALC@line}{-1}
  \refstepcounter{ALC@line}
  \label
}
\newcommand{
  \gitrepo
}{{\footnotesize\url{https://github.com/aidagroup/calf-enhance}}\xspace}

\providecommand{\Value}{\mathcal{V}}
\providecommand{\QValue}{Q}
\makeatletter
\providecommand{\leftsquigarrow}{\mathrel{\mathpalette\reflect@squig\relax}}
\newcommand{\reflect@squig}[2]{\reflectbox{$\m@th#1\rightsquigarrow$}}
\makeatother

\begin{document}

\twocolumn[
  \begin{@twocolumnfalse}

\title{An Agency-Transferring Model-Free Policy Enhancement Technique}

\author{
  Anton Bolychev\thanks{The first two authors contributed equally.}\\
  \small Center for Engineering Systems and Sciences\\
  \small \texttt{bolychev.anton@gmail.com}
  \And
  Georgiy Malaniya\footnotemark[1]\\
  \small Center for Engineering Systems and Sciences\\
  \small \texttt{g.malaniya@rcdei.ru}
  \And
  Sinan Ibrahim\\
  \small Center for Engineering Systems and Sciences\\
  \small \texttt{s.ibrahim@rcdei.ru}
  \And
  Pavel Osinenko\thanks{Corresponding author.}\\
  \small Central University; Center for Engineering Systems and Sciences;\\
  \small Sirius University of Science and Technology\\
  \small \texttt{p.osinenko@gmail.com}
}

\date{}

\maketitle

\begin{abstract}
  Training reinforcement learning (RL) policies from scratch is
  costly: it requires careful reward and environment design,
  extensive tuning, and substantial computation.
  Yet many control problems already have a functional but
  suboptimal policy available as a baseline.
  This paper proposes a method for embedding such a baseline into
  the RL training process, simultaneously improving training
  efficiency relative to from-scratch methods and producing a
  learning policy that outperforms the baseline.
  At each step, the method arbitrates between the baseline policy
  and a trainable learning policy, initially relying strongly on
  the baseline policy and then progressively transferring agency to
  the learning policy.
  By the end of training, the learning policy is a standalone
  neural network that operates without baseline policy support.
  The paper formalizes what it means for the baseline policy to be
  functional: under this policy, the agent reaches a goal set and
  remains there with high probability.
  The proposed arbitration mechanism is designed to exploit this
  property during training, yielding high goal-reaching rates right
  from the beginning of training.
  A theoretical analysis provides a formal interpretation of this
  behavior under stated assumptions and extends it to the final
  baseline-free regime, where explicit lower bounds are derived for
  the goal-reaching probability of the standalone learning policy.
  Empirical results on continuous-control benchmarks show that the
  proposed method achieves returns that match or exceed those of
  competitive approaches, while maintaining the highest
  goal-reaching rates throughout training among the compared
  methods---including in the final stage, where the learning policy
  operates without any baseline support.
\end{abstract}

\keywords{Reinforcement learning \and Policy arbitration \and Policy switching}

\vspace{0.35cm}

  \end{@twocolumnfalse}
]

\begingroup
\renewcommand{\thefootnote}{\fnsymbol{footnote}}
\footnotetext[1]{The first two authors contributed equally.}
\footnotetext[2]{Corresponding author.}
\endgroup

\section{Introduction}
\label{sec:introduction}

Reinforcement learning (RL) has repeatedly demonstrated that a single
framework can master an impressive spectrum of complex control problems.
Flagship achievements include AlphaGo and AlphaZero's dominance in
Go, chess, and shogi~\cite{Silver2018generalreinfor}; OpenAI Five's
grand-master-level play in Dota~2~\cite{openai2019dota2largescale};
and AlphaStar's success in StarCraft II~\cite{Vinyals2019Grandmasterlev}.
In robotics, RL enabled a dexterous hand to solve a Rubik's
Cube~\cite{Akkaya2019Solvingrubiks} and has powered a variety of
manipulation and locomotion systems~\cite{Surmann2020DeepReinforcem}.
Yet training an RL agent from scratch is still tricky.
The outcome hinges on a number of small decisions --- how one clips
rewards, sets learning rates, or normalizes inputs --- and on a
toolkit of hard-won tricks~\cite{Engstrom2020Implementation}.
Libraries such as Stable-Baselines3~\cite{stable-baselines3} and
CleanRL~\cite{huang2022cleanrl} bundle these tricks, but they do not
lift all the weight.
One still needs to pick the right hyperparameters, design the
training environment carefully \cite{Eimer2023Hyperparameters}, and
secure enough compute for long runs on high-end hardware.
RL is powerful, yet getting to a working implementation can be
time-consuming and may require a great deal of effort even for an
experienced practitioner.

These practical difficulties, together with the effectiveness of
reinforcement learning itself, constitutes the central motivation of
the present paper.
A class of control problems is considered in which a functional
policy is available as a baseline that completes the task but does
not achieve the desired performance.
Rather than training a reinforcement learning policy from scratch,
this paper proposes a method for embedding prior knowledge, in the
form of the baseline policy, into the RL training process.
This approach simultaneously improves training efficiency compared to
vanilla from-scratch methods such as SAC~\cite{Haarnoja2018},
PPO~\cite{Schulman2017ProximalPolicy}, and TD3~\cite{Fujimoto2018},
and yields a policy that outperforms the baseline.

Problem settings with an available but suboptimal baseline policy
arise naturally across many application domains.
A logistics company may rely on heuristic routing algorithms and seek
to improve delivery efficiency; a financial trader may aim to refine
a profitable but coarse rule-based strategy; a robotics engineer may
use a controller that ensures stability but fails to optimize energy
usage; or a game AI developer may start with rule-based agents and
strive for superhuman performance in complex multi-agent settings.
When accurate modeling is difficult, as in financial trading, vehicle
routing, or video game AI, the available policy may come from
hand-crafted routines, engineering heuristics, textbook techniques,
or other application-specific sources.
Such solutions are unlikely to maximize metrics of interest, just as
in related approaches that also embed a baseline policy into the
learning process the design of that baseline policy is not the
primary concern, with the focus instead placed on how it is
integrated into training.
One such related approach is residual reinforcement learning
(residual RL)~\cite{johannink2018residualreinforcementlearningrobot,
  silver2019residualpolicylearning,
  alakuijala2021residualreinforcementlearningdemonstrations,
Sheng2024Trafficexperti}, which serves as the closest comparison
baseline in this paper and is discussed in more detail in
\Cref{sec:related_work}.

The presence of approaches such as residual RL highlights that
embedding prior knowledge via a pre-existing baseline policy falls
within a well-established class of reinforcement learning methods.
The contribution of the present paper therefore lies in the
development of a reinforcement learning method that leverages a
pre-existing baseline policy, together with its empirical validation
within an established and well-studied problem setting.

\subsection{High-level overview of the proposed method}

In standard reinforcement learning, an agent interacts with an
environment over discrete time steps.
At step $t$, the agent observes the current state $\State_t$ and
selects an action $\Action_t$ according to its policy $\pi(\bullet
\mid \State_t)$, which is formally a conditional probability
distribution over the action space.
The environment executes that action, returns a reward, and produces
the next state.
The overall objective is to maximize cumulative reward.
This objective can be pursued with a variety of standard RL
algorithms, including canonical methods such as
SAC~\cite{Haarnoja2018}, PPO~\cite{Schulman2017ProximalPolicy}, and
TD3~\cite{Fujimoto2018}.
The more formal setup used here is introduced later in \Cref{sec:rl-objective}.
A schematic overview of this standard interaction loop is depicted in
\Cref{fig:rl_interaction_loop} below.

\begin{center}
  \includegraphics[width=0.88\linewidth]{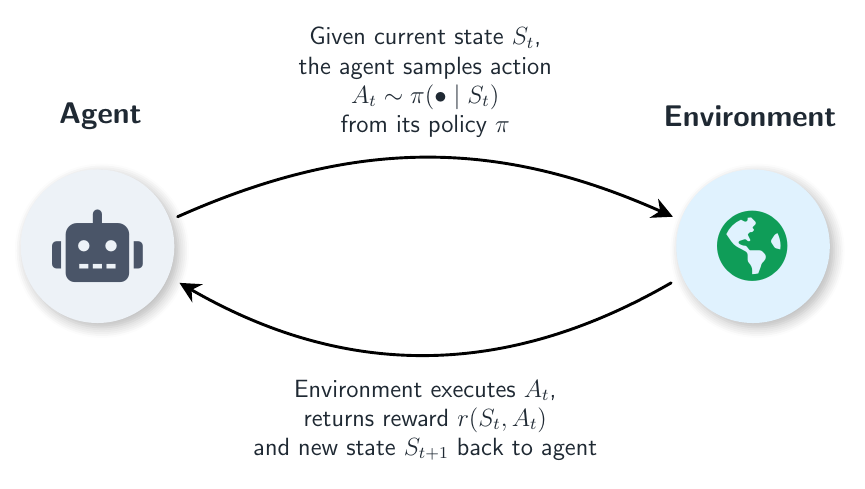}
  \captionof{figure}{ Standard reinforcement learning interaction loop.
    The agent aims to identify a policy that maximizes the expected
    discounted cumulative reward $\sum_{t=0}^{\infty} \gamma^t
    \reward(\State_t, \Action_t)$, where $\gamma \in (0, 1)$ is the
    discount factor and a hyperparameter of the problem.
    The formal problem statement is given in \Cref{sec:rl-objective}
    and \eqref{eq:objective}.
  }
  \label{fig:rl_interaction_loop}
\end{center}

The proposed method can be built on top of any appropriate RL
backbone algorithm, including SAC~\cite{Haarnoja2018},
PPO~\cite{Schulman2017ProximalPolicy}, and TD3~\cite{Fujimoto2018}.
At a high level, such backbone algorithms improve a policy through
repeated interaction with the environment.
They collect transition data generated by the agent and use these
data to train an auxiliary model, commonly called a critic, whose
role is to estimate how favorable particular actions or states are in
terms of expected future reward.
The policy is then updated in a direction favored by these estimates,
so that the critic provides the learning signal used to improve decision making.
Accordingly, the backbone fully determines the updates of the
learning policy, critic, and any other trainable quantities.
The proposed method leaves this learning mechanism intact and
augments it with an arbitration module that governs two policies: the
\textit{learning policy}, represented by the trainable policy
network, and a \textit{baseline policy}.
At each time step, the arbitration module determines whether the
executed action is supplied by the learning policy or by the baseline policy.
Initially, it favors the baseline policy.
As training progresses, it increasingly favors the learning policy,
gradually reducing dependence on the baseline.
Upon training completion, the learning policy, represented by a
standalone neural network, operates independently.


Formally, at each training step $t$, the arbitration module can be
represented by a mixing coefficient $\alpha_t \in [0, 1]$.
This coefficient determines how the action $\Action_t$ executed at
step $t$ is sampled from a probabilistic mixture of the baseline
policy $\baselinepolicy$ and the learning policy $\learningpolicy[t]$
with weights $\theta_t$ at time step $t$:
\begin{equation}\label{eq:calf-action-mixture-distribution}
  \Action_t \sim (1 - \alpha_t) \, \baselinepolicy(\bullet
  \mid \State_t)
  + \alpha_t \, \learningpolicy[t](\bullet \mid \State_t),
\end{equation}
where both $\baselinepolicy(\bullet \mid \State_t)$ and
$\learningpolicy[t](\bullet \mid \State_t)$ are, in general,
stochastic distributions over the action space conditioned on the
current state $\State_t$.
The key design problem of the proposed method is therefore the
construction of the coefficient $\alpha_t$, which in effect
constitutes the arbitration module itself.
The coefficient $\alpha_t$ can be represented as a Bernoulli random variable.
It is equal to one almost surely if the critic value of the candidate
learning action exceeds the current episode-local critic value
$Q_t^{\dagger}$ by the margin $\nu$.
Here, $Q_t^{\dagger}$ is the maximum critic value observed in the
current episode up to the margin $\nu$, where $\nu$ is also an
algorithmic hyperparameter.
Here and below, this critic is denoted by $Q^w$, where $w$ collects
its trainable parameters.
Otherwise, it is sampled with success probability
$\relprob\lambda^{t-\tau}$, where $\tau$ is the beginning time of the
current episode and $t$ is the current training time step.
Equivalently, $\alpha_t$ is obtained by an ordinary logical OR
between the critic-trigger event and the random relaxation event:
\[
  \alpha_t = \mathbb{I}\!\left[
    Q^w(\State_t,\Action_t^{\learningpolicy}) \ge Q_t^\dagger+\nu
  \;\lor\; U_t \le \relprob\lambda^{t-\tau} \right].
\]
Here, $U_t \sim \mathrm{Uniform}[0,1]$ and $\mathbb{I}$ denotes the
indicator function.
The parameters $\relprob$ and $\lambda$ are fixed within an episode
and updated only between episodes, increasing monotonically until
both become equal to one.
A high-level overview is shown in \Cref{fig:diagram}, and a detailed
description is given in \Cref{sec:approach}; see in particular
\Cref{sec:arbitration-module} and Algorithm~\ref{alg:enhance}.
The activation probability $\PP{\alpha_t = 1}$ is designed to remain
relatively small on average during the early stages of training, so
that most actions are still selected from the baseline policy.
As training progresses, this activation probability is designed to
increase on average according to a prescribed schedule until, after a
finite transition time $T_{\mathrm{tran}}$, it becomes equal to $1$.
This terminal case corresponds to both $\relprob$ and $\lambda$
reaching their maximum possible value of $1$, so that $\alpha_t = 1$
almost surely.
As a result, all subsequent actions are selected from the learning
policy and the baseline policy is no longer invoked.

\begin{figure}[!t]
  \includegraphics[width=\linewidth, trim=15mm 0 15mm 0,
  clip]{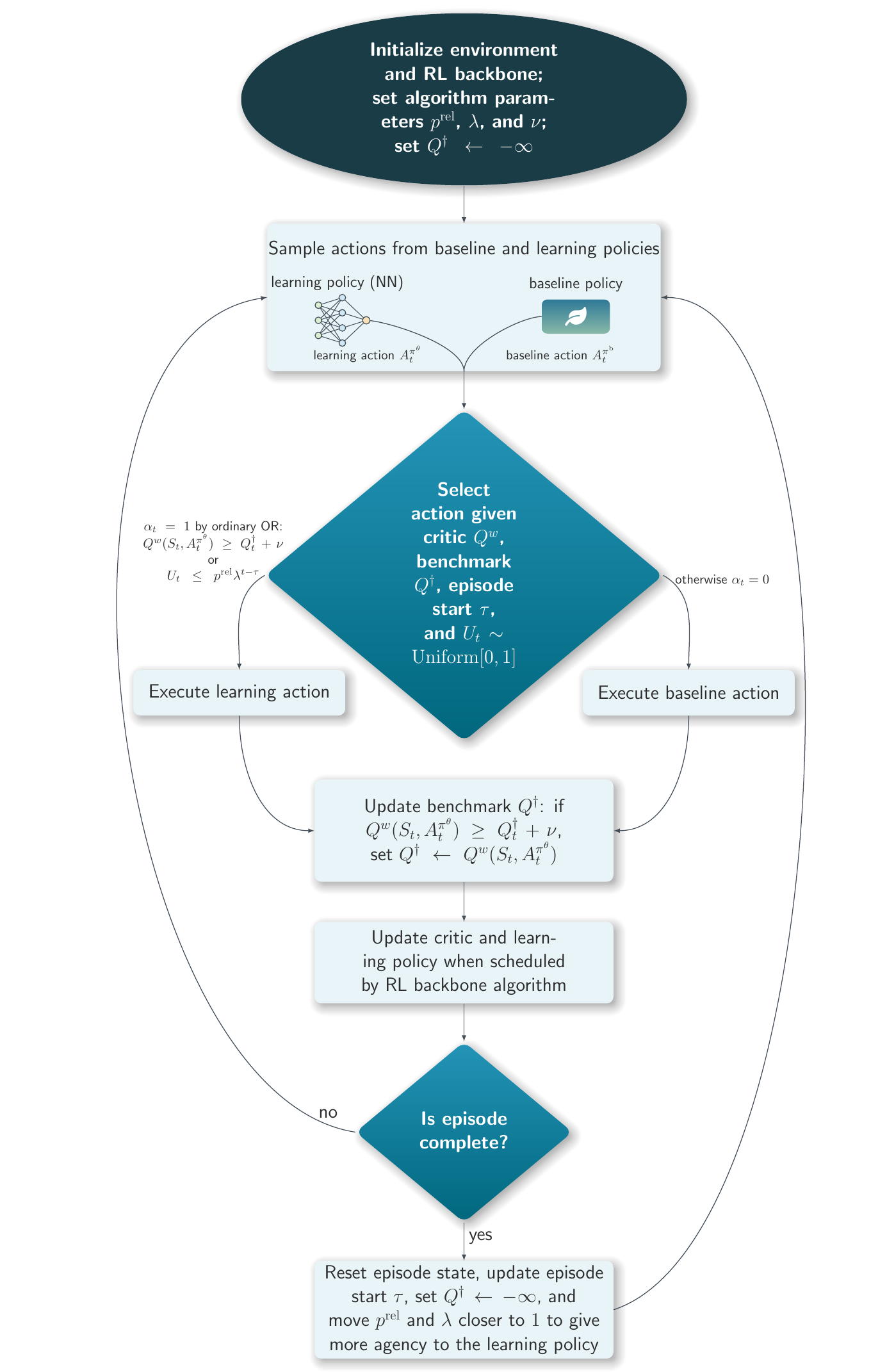}
  \caption{ Diagram of the proposed method.
    The diagram shows action sampling from the baseline and learning
    policies, the arbitration rule, the episode-local benchmark
    update, and the scheduled updates of the RL backbone.
    A detailed version of the algorithm illustrated here is provided
    in \Cref{alg:enhance} within \Cref{sec:approach}.
  }
  \label{fig:diagram}
\end{figure}

In tasks where solving the environment amounts to driving the agent
to a goal set and keeping it there, the proposed method offers an
additional advantage.
In this setting, the design of the mixing coefficient $\alpha_t$
allows the method to retain the beneficial influence of the baseline
policy from the very beginning of training, provided that the
baseline policy itself is able to reach the goal, albeit in a
generally suboptimal manner.
Experimental findings in \Cref{sec:exp_results} indicate that,
compared with the other evaluated algorithms, this is associated with
consistently high goal-reaching rates throughout training.
Moreover, under the stated assumptions, this phenomenon admits a
theoretical interpretation of why it may happen that goal-reaching
performance remains strong, as observed in the experimental findings,
during the early and intermediate stages of learning (i.e., for $t <
T_{\mathrm{tran}}$); see \Cref{sec:theoretical_analysis}.
The uniform result in the main body strengthens this analysis under
additional assumptions on the baseline policy that quantify how
quickly trajectories approach the goal set from a given initial
state; see Definition~\ref{dfn:uniform_eta_improbable} and
\Cref{thm:uniform_goal_reaching}.

However, the baseline policy is fully phased out by design once the
transition time $T_{\mathrm{tran}}$ is reached.
Another question addressed here is whether such performance persists
once the learning policy operates autonomously (i.e., for $t \geq
T_{\mathrm{tran}}$).
Experimental findings in \Cref{sec:exp_results} suggest that it does,
with the proposed method attaining the highest goal-reaching rates
among the compared approaches.
This question is also addressed formally in \ref{sec:transfer}
through a transfer analysis between two regimes: the
baseline-assisted regime, in which baseline actions may still be
selected, and the baseline-free regime, in which actions are
generated exclusively by the learning policy.
Under the stated assumptions, this analysis yields an explicit lower
bound on the probability that the standalone learning policy reaches
the goal region and remains there.
In qualitative terms, the result shows that some performance
degradation may occur once the baseline is removed, but that this
degradation is not arbitrary and can be characterized within the
proposed framework.
The analysis therefore clarifies the conditions under which the loss
after full transfer of agency remains small, and when the standalone
learning policy can still be expected to retain strong goal-reaching behavior.

\subsection{Related work}\label{sec:related_work}

A number of methods leverage existing policies into the training process.
The existing policy can be available in two forms: as recorded traces
(or \emph{demonstrations}) of its behavior (trajectories,
state--action pairs, videos, etc.), or as a callable policy that can
be queried on the environment at any time step.

A large body of work operates in the first regime.
Many well-established methods do not follow the classical RL setting
at all, instead seeking a policy that reproduces or improves upon
existing demonstrations of expert behavior through supervised or
self-supervised objectives.
VPT~\cite{Baker2022VPT} recovers policies from unlabeled internet
video via an inverse-dynamics model.
GAIL~\cite{Ho2016GAIL} does employ RL, but only as an intermediate
mechanism: the core learning signal comes from a discriminator
trained to distinguish expert from generated behavior, with RL
serving to optimize the policy against this surrogate reward.
A separate line of work takes full RL training as the backbone and
injects demonstrations directly into it.
DQfD~\cite{Hester2018DQfD} permanently seeds the replay buffer with
expert transitions, conducts an offline warm-up phase, and augments
the loss with supervised classification and regularization terms;
DAPG~\cite{Rajeswaran2018DAPG} pre-trains the policy via behavioral
cloning and adds a geometrically decaying demonstration-weighted
auxiliary loss to the policy gradient.

All of the methods above operate on a fixed, pre-collected dataset of
demonstrations.
The proposed approach belongs to a different regime: prior knowledge
is available as a callable baseline policy that can be queried at
every time step during training, and the method acts at the level of
action selection without modifying the backbone RL algorithm's loss
functions or gradients.
Methods that share this live-policy regime are therefore the
appropriate point of comparison.
A widely adopted representative of this regime is residual
reinforcement learning (residual
RL)~\cite{johannink2018residualreinforcementlearningrobot,
  silver2019residualpolicylearning,
  alakuijala2021residualreinforcementlearningdemonstrations,
Sheng2024Trafficexperti}.

In residual RL, the control action is decomposed into the algebraic
sum of a pre-designed controller (i.e., the baseline policy) and a
learned residual policy trained via reinforcement learning.
Formally, given the current state $\State_t$, the executed action
$\Action_t$ is defined as
\begin{equation}
  \label{eq:residual_rl}
  \Action_t = \Action_t^{\baselinepolicy} + \Action_t^{\learningpolicy},
\end{equation}
where $\Action_t^{\baselinepolicy} \sim \baselinepolicy(\bullet \mid
\State_t)$ is sampled from the baseline policy $\baselinepolicy$,
which is generally a probability distribution over the action space
conditioned on the current state $\State_t$, and
$\Action_t^{\learningpolicy} \sim \learningpolicy[t](\bullet \mid
\State_t)$ is sampled from the learning (residual) policy
$\learningpolicy[t]$ with parameters $\theta_t$ at time step $t$.

Residual RL has demonstrated superior performance, particularly in
domains of complex robotic manipulation.
Experiments show that learning a residual on top of hand-designed or
model-predictive controllers yields substantial improvements in
challenging manipulation tasks involving partial observability,
sensor noise, model misspecification, and controller
miscalibration~\cite{silver2019residualpolicylearning}.

Complementary work applies residual RL to real-world robot
manipulation with contact and friction
dynamics~\cite{johannink2018residualreinforcementlearningrobot},
extends the paradigm to learning from demonstrations with
high-dimensional visual
inputs~\cite{alakuijala2021residualreinforcementlearningdemonstrations},
and demonstrates generalization to new challenges such as autonomous
vehicle control by combining prior knowledge with residual
learning~\cite{Sheng2024Trafficexperti}.

\begin{table*}[!t]
  \centering \scriptsize \captionsetup{width=0.88\textwidth}
  \caption{
    Functional comparison of related method families.
    The first four columns report required resources and deployed
    dependencies: ``Yes'' indicates that the resource is usually
    required, and a centered dot indicates that it is not.
    The final column reports theoretical support: a check mark
    indicates that formal closed-loop behavior analysis is provided.
  }
  \label{tab:related_work_comparison}
  \setlength{\tabcolsep}{5pt}
  \begin{tabular}{@{}lccccc@{}}
    \toprule & \multicolumn{4}{c}{Required resources and
    dependencies} & \multicolumn{1}{c}{Theoretical support}
    \\ \cmidrule(lr){2-5}\cmidrule(l){6-6} Method family &
    \makecell{Recorded\\external\\behavior} & \makecell{Callable
    external\\policy/controller\\during training} &
    \makecell{External policy/\\controller\\at deployment} &
    \makecell{Model or\\constraint oracle} & \makecell{Formal
      analysis of\\closed-loop
    behavior\\(stability/safety/\\reachability/constraints)}
    \\ \midrule Standard model-free
    RL~\cite{Haarnoja2018,Fujimoto2018,Schulman2017ProximalPolicy} &
    $\cdotp$ & $\cdotp$ & $\cdotp$ & $\cdotp$ & $\cdotp$
    \\ Demonstration-based
    imitation/RL~\cite{Baker2022VPT,Ho2016GAIL,Hester2018DQfD,Rajeswaran2018DAPG}
    & Yes & $\cdotp$ & $\cdotp$ & $\cdotp$ & $\cdotp$ \\ Safe,
    shielded, CBF/CLF, and switching
    control~\cite{Garcia2015SafeRLSurvey,Achiam2017CPO,Chow2018LyapunovSafeRL,Alshiekh2018Shielding,Ames2017CBF,Liberzon2003Switching,Branicky1998MultipleLyapunov}
    & $\cdotp$ & Yes & Yes & Yes & $\checkmark$ \\ Residual
    RL~\cite{johannink2018residualreinforcementlearningrobot,silver2019residualpolicylearning,alakuijala2021residualreinforcementlearningdemonstrations,Sheng2024Trafficexperti}
    & $\cdotp$ & Yes & Yes & $\cdotp$ & $\cdotp$ \\ Proposed method &
    $\cdotp$ & Yes & $\cdotp$ & $\cdotp$ & $\checkmark$ \\ \bottomrule
  \end{tabular}
\end{table*}

\subsubsection{Comparison to Residual Reinforcement
Learning}\label{sec:comparison_to_residual_rl}

Residual RL shares several key methodological similarities with the
proposed method, making it a fair and meaningful point of comparison:
\begin{itemize}
  \item it requires a comparable set of prerequisites.
    Specifically, it explicitly embeds a baseline policy into the
    training process;
  \item it can be instantiated on top of any widely used off-policy
    actor--critic reinforcement learning algorithm (similarly, the
      proposed method can be applied on top of vanilla RL algorithms
      such as SAC~\cite{Haarnoja2018},
    PPO~\cite{Schulman2017ProximalPolicy}, and TD3~\cite{Fujimoto2018});
  \item it has an implementation complexity comparable to the
    proposed method, as both approaches require only a minimal
    modification of the action computation while leaving the core
    training components of the underlying reinforcement learning
    algorithm (i.e., the critic and policy update routines) unchanged.
\end{itemize}

However, it is important to emphasize that residual RL is different
from the proposed method in several key ways:
\begin{itemize}
  \item during deployment, Residual RL continues to rely on the
    baseline policy, whereas in the proposed approach the baseline is
    completely removed through a deliberately designed arbitration
    module that transfers agency from the baseline to the learning policy;
  \item Residual RL computes actions as an algebraic sum of the
    baseline and residual policies:
    \begin{equation}
      \Action_t = \Action_t^{\baselinepolicy} +
      \Action_t^{\learningpolicy}, \text{ where }
      \Action_t^{\baselinepolicy} \sim \baselinepolicy(\bullet \mid
      \State_t),\; \Action_t^{\learningpolicy} \sim
      \learningpolicy(\bullet \mid \State_t).
    \end{equation}
    In contrast, the proposed method employs a switching mechanism in
    which the executed action is sampled from an algebraic mixture of
    two action distributions corresponding to the baseline and the
    learning policy:
    \[
      \Action_t \sim (1 - \alpha_t) \, \baselinepolicy(\bullet \mid
      \State_t) + \alpha_t \, \learningpolicy[t](\bullet \mid \State_t),
    \]

  \item as shown in \Cref{sec:experiments}, the proposed method
    preserves high goal-reaching rates throughout training, including
    the initial and intermediate stages of learning.
    In
    \Cref{sec:theoretical_analysis}, theoretical insight is provided
    into this phenomenon.
    In contrast, the canonical formulations of Residual RL and
    from-scratch reinforcement learning algorithms such as
    PPO~\cite{Schulman2017ProximalPolicy}, TD3~\cite{Fujimoto2018},
    and SAC~\cite{Haarnoja2018} were not originally designed to admit
    theoretical interpretations of this kind.
    Moreover, in practice, their performance during the initial and
    intermediate stages of training is often poor, as the underlying
    learning policies remain significantly undertrained in these stages.

\end{itemize}

\subsubsection{Safe, Shielded, and Switching Control Constitute a
Different Problem Setting} The proposed method differs fundamentally
from safe reinforcement learning, shielded RL, and classical
switching controllers.
Safe and shielded RL approaches introduce persistent safety
mechanisms---such as hard constraints, safety critics, or action
filters---that remain active throughout deployment (and/or training)
to enforce constraint satisfaction~\cite{Garcia2015SafeRLSurvey,
  Achiam2017CPO, Chow2018LyapunovSafeRL,
Bharadhwaj2021ConservativeSafetyCritics, Dalal2018SafeExploration,Ames2017CBF}.
In contrast, the proposed method does not impose safety constraints
nor guarantee constraint satisfaction; instead, it leverages a
functional baseline policy as a \emph{temporary} competence scaffold
during training.
Unlike shielded reinforcement learning, where unsafe actions are
deterministically overridden at runtime by a formally synthesized
\emph{shield}~\cite{Alshiekh2018Shielding}, the result of the
proposed approach is a standalone neural network, and no actions are
overridden during execution.
Finally, while switching and hybrid control rely on pre-defined
state-dependent or rule-based mode
selection~\cite{Liberzon2003Switching, Branicky1998MultipleLyapunov,
Morse1996SupervisoryPartI, Morse1997SupervisoryPartII}, the proposed
method performs learning-adaptive arbitration driven by critic
improvement and probabilistic scheduling, enabling a smooth,
data-driven handover of control authority rather than a fixed or
permanent switching structure.

\subsubsection{Comparative Summary}

The related approaches discussed above can be summarized by five
functional features that determine how external competence enters the
learning problem and whether it remains part of the final controller.
\Cref{tab:related_work_comparison} uses these features to place the
proposed method relative to the main families of related work.
The table separates required resources and dependencies from
theoretical support.
The columns have the following interpretation:
\begin{itemize}
  \item \emph{Recorded external behavior} marks methods that use
    demonstrations, traces, logs, or offline trajectories from another source.
    This feature captures approaches where prior competence is
    present only through data and the generating policy or controller
    cannot be queried during training.
  \item \emph{Callable external policy/controller during training}
    marks methods that can query an external policy, controller,
    shield, safety filter, or prior while training the learning policy.
  \item \emph{External policy/controller required at deployment}
    marks methods whose deployed system still needs that external
    component after training.
  \item \emph{Model or constraint oracle required} marks methods that
    rely on dynamics, a learned model, a constraint checker, a
    control barrier or Lyapunov function, a reachability oracle, or a
    comparable external formal object.
  \item \emph{Formal analysis of closed-loop behavior} marks method
    families that provide an analysis of stability, safety,
    constraint satisfaction, recoverability, reachability, or
    goal-reaching behavior under the induced policy or controller.
    Generic optimization or convergence analyses are not counted here
    unless they characterize the behavior of the controlled closed-loop system.
\end{itemize}

\section{Problem Statement and Notation}

\subsection{Markov Decision Process}

The environment is modeled as a Markov Decision Process (MDP) defined
by the tuple $(\states, \actions, \transit, \transit_0, \reward)$, where:
\begin{itemize}
  \item $\states$ is the state space (assumed to be a Banach space
    with norm $\norm{\bullet}$ for theoretical completeness);
  \item $\actions$ is the action space (assumed to be a compact
    topological space);
  \item $\transit(\bullet \mid \state, \action) : \states \times
    \states \times \actions \to \R_{\geq 0}$ is the transition
    probability density function for the sampling of the next state
    $\State'$ given the current state $\state$ and the action
    $\action$ \ie $\State' \sim \transit(\bullet \mid \state, \action)$.
    The theoretical analysis assumes the existence of an upper
    semi-continuous function $\bar{p} : \states \times \actions \to
    \mathbb{R}_{\geq 0}$ such that for any $\state \in \states$ and
    $\action \in \actions$, the next state $\State'$ sampled from
    $\transit(\bullet \mid \state, \action)$ satisfies
    $\PP{\norm{\State'} \leq \bar{p}(\state, \action)} = 1$,
    effectively bounding the system's one-step transition magnitude;
  \item $\transit_0(\bullet) : \states \to \R_{\geq 0}$ is the
    initial state distribution, where $\State_0 \sim \transit_0(\bullet)$;
  \item $\reward(\state, \action) : \states \times \actions \to
    \mathbb{R}$ is the reward function.
\end{itemize}

\subsection{Policy Definitions}

A distinction is drawn between two types of policies:

\begin{dfn}[Stationary Policy]
  A stationary policy $\policy(\bullet \mid \state) : \states \to
  \R_{\geq 0}$ maps each state $\state$ to a probability density over
  action space $\actions$, independent of time.
\end{dfn}

\begin{dfn}[Non-stationary Policy]
  A non-stationary policy $\policy(\bullet \mid \state, t) : \states
  \times \mathbb{N} \to \R_{\geq 0}$ maps each state-time pair
  $(\state, t)$ to a probability density over action space $\actions$.
\end{dfn}

Let $\policies_{\text{stat}}$ and $\policies_{\text{nstat}}$ denote
the sets of stationary and non-stationary policies, respectively.
Note that $\policies_{\text{stat}} \subseteq \policies_{\text{nstat}}$.

\subsection{Reinforcement Learning Objective}
\label{sec:rl-objective}

The standard RL objective seeks a policy $\policy \in
\policies_{\text{nstat}}$ that maximizes the expected discounted return:
\begin{equation}
  \label{eq:objective}
  J(\policy) = \mathbb{E} \left[ \sum_{t=0}^{\infty} \gamma^t
  \reward(\State_t, \Action_t) \right],
\end{equation}
where $\gamma \in (0, 1]$ is the discount factor, and the trajectory
$\{(\State_t, \Action_t)\}_{t=0}^{\infty}$ is generated according to:
\begin{equation}
  \State_0 \sim \transit_0(\bullet), \;
  \Action_t \sim \policy(\bullet \mid \State_t, t), \;
  \State_{t+1} \sim \transit(\bullet \mid \State_t, \Action_t).
\end{equation}
For a fixed stationary policy $\policy \in \policies_{\text{stat}}$,
the state-value function is defined by
\begin{equation}
  \label{eq:value-function}
  \Value^{\policy}(\state)
  =
  \mathbb{E}
  \left[
    \sum_{t=0}^{\infty}
    \gamma^t
    \reward(\State_t, \Action_t)
    \mid
    \State_0 = \state
  \right].
\end{equation}
Similarly, the corresponding action-value function (Q-function) is defined by
\begin{equation}
  \label{eq:action-value-function}
  \QValue^{\policy}(\state, \action)
  =
  \mathbb{E}
  \left[
    \sum_{t=0}^{\infty}
    \gamma^t
    \reward(\State_t, \Action_t)
    \mid
    \State_0 = \state,
    \Action_0 = \action
  \right].
\end{equation}
In both definitions, the trajectory evolves according to $\Action_t
\sim \policy(\bullet \mid \State_t)$ and $\State_{t+1} \sim
\transit(\bullet \mid \State_t, \Action_t)$ after the stated initial condition.

\subsection{Goal Set and Goal-Reaching Property}

For theoretical analysis and experimental validation, the concept of
a goal set is introduced and the goal-reaching property is formalized.

\begin{dfn}[Goal Set]
  The goal set $\G \subseteq \states$ is a closed subset of the state
  space representing the desired target region.
  The environment is considered solved when the system reaches $\G$
  and stays there.
\end{dfn}

The goal set need not always be supplied as an independent object.
It may be specified directly, for instance as a target region, or it
may be induced by the reward function.
When the reward is maximal at a target state or target region and
decreases away from it, a suitable neighborhood of the reward
maximizer defines the goal region.
In this sense, the reward function itself can determine what it means
for the environment to be solved, while the goal-reaching property
formalized in Definition~\ref{dfn:goal_reaching_property} specifies
whether a policy reliably reaches that region.

\begin{generalboundednessassumption}
  \label{ass:reward-upper-bound}
  When the goal region is understood as the region where the task is
  solved, it would be unnatural for rewards away from any
  neighborhood of that region to be unbounded from above.
  Such rewards would make the objective favor behavior unrelated to
  solving the task.
  Thus, in goal-reaching problems of this type, it is natural to
  assume that the reward is bounded from above: there exists
  $r^{\max} < \infty$ such that $\reward(\state,\action) \le
  r^{\max}$ for all $(\state,\action) \in \states \times \actions$.
\end{generalboundednessassumption}

To analyze policy performance with respect to goal-reaching, the
following notation is introduced:

\begin{itemize}
  \item $\State^{\policy}_t(\state_0)$ denotes the state at time $t$
    under policy $\policy \in \policies_{\text{nstat}}$ with initial
    state $\State_0 = \state_0 \in \states$;
  \item $\goaldist(\state) = \inf_{\state' \in \G} \norm{\state -
    \state'}$ denotes the distance from state $\state$ to the goal set $\G$.
\end{itemize}

\begin{dfn}[$\varepsilon$-improbable goal-reaching property]
  A policy $\policy \in \policies_{\text{nstat}}$ satisfies the
  $\varepsilon$-goal-reaching property if for any initial state
  $\state_0 \in \states$, $$ \PP{ \lim_{t \to \infty}
  \goaldist(\State^{\policy}_t(\state_0)) = 0 } \geq 1 - \varepsilon,
  $$ where $\varepsilon \in [0, 1)$ is the failure probability tolerance.
  \label{dfn:goal_reaching_property}
\end{dfn}

This formalization makes it possible to demonstrate that the approach
described in \Cref{sec:approach} preserves high goal-reaching rates
during the initial stages of learning.

\begin{remark}[Practical Interpretation]
  The $\varepsilon$-improbable goal-reaching property is not required
  for applying the proposed approach.
  Rather, it identifies a class of baseline policies for which the
  empirically observed goal-reaching behavior of the proposed method
  can be related to the theoretical analysis developed below.
\end{remark}

\subsection{Additional Notation}
\begin{itemize}
  \item Throughout the paper, $\mathbb{I}\{\bullet\}$ denotes
    indicator function \ie
    \[
      \mathbb{I}\{\bullet\} \coloneqq
      \begin{cases}
        1, & \text{if } \bullet \text{ holds}, \\
        0, & \text{otherwise.}
      \end{cases}
    \]

  \item For any vector $x=(x_1,...,x_n) \in \mathbb{R}^n$, $\|x\|_2$
    denotes the Euclidean norm, $\|x\| \coloneqq \left(\sum_{i=1}^n
    x_i^2\right)^{1/2}$.

  \item For any subsets $X_1, X_2 \subseteq \mathbb{R}^n$,
    $d(X_1,X_2)$ denotes the Euclidean distance between the subsets,
    $d(X_1,X_2) \coloneqq \inf_{x_1 \in X_1, x_2 \in X_2}\| x_1 - x_2\|_2$.

  \item In the algorithm listings only, $\gets$ denotes deterministic
    assignment and $\leftsquigarrow$ denotes stochastic assignment.
    In particular, $X \leftsquigarrow \mathrm{pdf}(\bullet)$ means
    that $X$ is sampled from the specified density $\mathrm{pdf}(\bullet)$.

\end{itemize}

\section{Approach}

\label{sec:approach}

The approach leverages an existing baseline policy to guide the
training process while gradually transitioning to an autonomous learning policy.
Upon training completion, the learning policy becomes a standalone
neural network that operates independently without relying on the
baseline policy.
This autonomy is achieved by integrating the baseline policy into the
learning process of the backbone RL algorithm used to train the learning policy.
The core principle is to make extensive use of the baseline policy
during early training stages and progressively transfer agency to the
learning policy.

The proposed approach is RL-algorithm-agnostic and any standard RL
algorithm can be used as its backbone, including
TD3~\cite{Fujimoto2018}, PPO~\cite{Schulman2017ProximalPolicy},
DDPG~\cite{Lillicrap2016Continuous}, and SAC~\cite{Haarnoja2018}.
The experiments discussed in \Cref{sec:experiments} employ TD3 and
SAC; the TD3 backbone is also consistent with the implementation
choices in residual RL~\cite{johannink2018residualreinforcementlearningrobot}.


The remainder of this section is organized into several subsections,
each introducing one of the main constituents from which the proposed
method, formally specified in \Cref{alg:enhance}, is built.
\Cref{alg:enhance-intra-episode} should not be interpreted as a
separate algorithm executed independently of the proposed method: it
is invoked inside \Cref{alg:enhance}, line~\ref{alg:rollout-call},
and forms part of that algorithm.
This decomposition is adopted for clarity and to reduce the cognitive
load on the reader.

\subsection{Baseline Policy}
The baseline policy $\baselinepolicy \in \policies_{\text{stat}}$ is
a functional but suboptimal policy that successfully completes the task.
While it is assumed that $\baselinepolicy$ satisfies the
$\varepsilon$-improbable goal-reaching property for theoretical
analysis, this is not a strict requirement in practice.
The property serves to formalize the notion of a ``functional'' or
``working'' policy.

\subsection{Learning Policy}
The learning policy $\learningpolicy$ is a policy parameterized by
weights $\theta$ (e.g. artificial neural network) that is trained in
order to outperform the baseline policy.

\subsection{Backbone RL Algorithm}
The backbone RL algorithm is the underlying reinforcement learning
method used to train $\learningpolicy$.
\Cref{alg:enhance} is flexible and permits the usage of an arbitrary
RL algorithm as the backbone.
An RL backbone is characterized by when it performs critic and policy
updates and by the corresponding update operators.
Accordingly, it is assumed that each backbone is specified by four routines:
\begin{itemize}
  \item \texttt{is\_critic\_update\_time},
    \texttt{perform\_critic\_update} determine whether a critic
    update should be performed at time step $t$ and implement that
    update; they are invoked in \Cref{alg:enhance-intra-episode},
    lines \ref{alg:intra-critic-1}-\ref{alg:intra-critic-2},
  \item \texttt{is\_policy\_update\_time},
    \texttt{perform\_policy\_update} determine whether a policy
    update should be performed at time step $t$ and implement that
    update; they are invoked in \Cref{alg:enhance-intra-episode},
    lines \ref{alg:intra-policy-1}-\ref{alg:intra-policy-2}.
\end{itemize}

For formal notation, the critic weights are denoted by $w$ and the
learning policy weights by $\theta$.
\Cref{alg:enhance-intra-episode} and \Cref{alg:enhance} are written
for a backbone with a $Q$-function-based critic.
The value-function version is obtained by a direct notational
replacement, as discussed in Remark~\ref{rem:value-function-critic-version}.
Accordingly, the critic approximation at training time step $t$ is
denoted by $Q^{w_t}(\state, \action)$, where $w_t \in \mathbb{W}$
represents the critic weights at time step $t$ and $\mathbb{W}$ is
the space of all critic weights.
The learning policy at time step $t$ is similarly denoted as
$\learningpolicy[t]$.

\begin{remark}[Upper-consistent critic]
  \label{rem:upper-consistent-critic}
  For a feasible and well-posed discounted-control problem, the
  objective functional in \eqref{eq:objective} has a maximal
  theoretically attainable value.
  Denote this scalar value by $Q^{\star}$.
  This value gives the theoretical upper scale for any policy-induced
  critic value.
  From a theoretical standpoint, it is therefore natural to
  parameterize the critic approximation class used during training so
  that every admissible critic is uniformly bounded from above by this scale.
  Since $Q^{\star}$ is generally unknown, a computable upper scale
  can be chosen from Assumption~\ref{ass:reward-upper-bound}:
  $Q^{\star} \le \bar Q := \sum_{k=0}^{\infty}\gamma^k r^{\max}$.
  This can be enforced, for example, by clipping critic outputs from
  above at $\bar Q$ or by using a final activation of the form $\bar Q - |x|$.
  The enforced critic then satisfies $Q^w(\state,\action) \le \bar Q$
  for all $(\state,\action,w)$.
\end{remark}

\subsection{Hyperparameters}
\label{sec:hyperparameters}

The algorithm makes use of four hyperparameters beyond those of the backbone:
\begin{itemize}
  \item $\relprob_0 \in [0, 1]$ (initial): base probability of
    selecting the learning policy over the baseline policy,
  \item $\lambda_{0} \in (0, 1]$ (initial): decay factor that reduces
    reliance on the baseline policy within episodes,
  \item $\nu > 0$: minimum improvement threshold required for
    updating $Q_t^{\dagger}$ to a new maximum, where $Q_t^{\dagger}$
    is the best critic value observed so far in the current episode,
  \item $T_{\text{tran}}$: the number of steps it takes to fully
    transfer agency to the learning policy,
\end{itemize}

\subsection{Arbitration Module}
\label{sec:arbitration-module}

The arbitration module is implemented in
\Cref{alg:enhance-intra-episode}, lines
\ref{alg:selection-1}--\ref{alg:arbitration-end}.

Let $t$ denote the current training time step and $\tau$ the time
when the current episode began.
Thus, $t - \tau + 1$ is the number of time steps that have passed
since the beginning of the current episode.
The maximum critic value observed in the current episode
$Q_t^{\dagger}$ is tracked online and initialized at the beginning of
every episode as $Q_{\tau}^{\dagger} = -\infty$.

At each time step $t$, the algorithm computes a candidate learning
action $\Action^{\learningpolicy}_t \leftsquigarrow
\learningpolicy[t](\bullet \mid \State_t)$ and a candidate baseline
action $\Action^{\baselinepolicy}_t \leftsquigarrow
\baselinepolicy(\bullet \mid \State_t)$ (lines
\ref{alg:selection-1}--\ref{alg:selection-1b}).
It selects between them by accepting the learning action whenever
\[
  Q^{w_t}(\State_t,\Action^{\learningpolicy}_t)\ge Q_t^{\dagger}+\nu
  \quad \text{or} \quad U_t \le \relprob\,\lambda^{t-\tau},
\]
and otherwise falling back to the baseline action (line
\ref{alg:selection-2}), where $U_t \sim \text{Uniform}[0,1]$.
The next state is then obtained as $\State_{t+1} \leftsquigarrow
\transit(\bullet \mid \State_t, \Action_t)$.

Finally, the episode-local benchmark is updated as follows (lines
\ref{alg:learning-3}--\ref{alg:learning-4}): if the learning action
exceeds the current benchmark by at least $\nu$, i.e.,
$Q^{w_t}(\State_t, \Action^{\learningpolicy}_t) \geq Q_t^{\dagger} +
\nu$, then $Q_{t+1}^{\dagger}$ is set to $Q^{w_t}(\State_t,
\Action^{\learningpolicy}_t)$; otherwise $Q_{t+1}^{\dagger} \gets
Q_t^{\dagger}$.

\paragraph{Intuition}\label{par:intuition}
The purpose of the arbitration module is to drive the trajectory
toward directions in which the critic estimate increases.
At time $t$, the quantity
$Q^{w_t}\!\left(\State_t,\Action^{\learningpolicy}_t\right)$ is
interpreted as an estimate of expected return for taking the learning
action in the current state.
Therefore, when
\[
  Q^{w_t}\!\left(\State_t,\Action^{\learningpolicy}_t\right) \ge
  Q_t^{\dagger} + \nu,
\]
the critic increase is treated as significant, which is interpreted
as a strong signal that the agent is moving in a promising direction;
the learning action in this case is accepted deterministically.
If the condition is violated, the action can still be accepted with
within-episode decaying probability $\relprob\lambda^{t-\tau}$.
This mechanism can be interpreted as controlled risk-taking, enabling
potential performance gains later in the episode.

This acceptance logic is directly analogous to simulated annealing
neighbor selection: improving candidates are accepted
deterministically, while non-improving candidates may still be
accepted with a probability that decreases as the temperature is
lowered (i.e., over the annealing schedule).

\begin{algorithm}[H]
  \caption{Intra-Episode Training}
  \label{alg:enhance-intra-episode}
  \begin{algorithmic}[1]
    \STATE \textbf{Input:} global time $t$, state $\State_t$, critic
    weights $w_t$, learning-policy weights $\theta_t$, baseline
    policy $\baselinepolicy$, schedule parameters $\relprob,\lambda$,
    threshold $\nu$.
    \STATE $\tau \gets t$, $Q^{\dagger}_t \gets -\infty$
    \REPEAT
    \STATE $\Action^{\learningpolicy}_t \leftsquigarrow
    \learningpolicy[t](\bullet \mid \State_t)$
    \alglinelabel{alg:selection-1}
    \STATE $\Action^{\baselinepolicy}_t \leftsquigarrow
    \baselinepolicy(\bullet \mid \State_t)$
    \alglinelabel{alg:selection-1b}
    \STATE $U_t \leftsquigarrow \mathrm{Uniform}[0, 1]$
    \IF{$Q^{w_t}(\State_t, \Action^{\learningpolicy}_t) \geq
      Q^{\dagger}_{t} + \nu$ or $U_t \leq \relprob \lambda^{t -
    \tau}$}
    \alglinelabel{alg:selection-2}
    \STATE $\Action_t \gets \Action^{\learningpolicy}_t$
    \ELSE
    \STATE $\Action_t \gets \Action^{\baselinepolicy}_t$
    \ENDIF
    \alglinelabel{alg:arbitration-end}
    \STATE $\State_{t+1} \leftsquigarrow \transit(\bullet \mid
    \State_t, \Action_t)$
    \STATE Store the transition for the backbone RL algorithm.
    \STATE $Q^{\dagger}_{t+1} \gets Q^{\dagger}_t$, $w_{t+1} \gets
    w_t$, $\theta_{t+1} \gets \theta_t$
    \IF{$Q^{w_t}(\State_t, \Action^{\learningpolicy}_t) \geq
    Q^{\dagger}_{t} + \nu$}
    \alglinelabel{alg:learning-3}
    \STATE $Q^{\dagger}_{t+1} \gets Q^{w_t}(\State_t,
    \Action^{\learningpolicy}_t)$
    \ENDIF
    \alglinelabel{alg:learning-4}
    \IF{\texttt{is\_critic\_update\_time}$(t)$}
    \alglinelabel{alg:intra-critic-1}
    \STATE $w_{t+1} \gets$ \texttt{perform\_critic\_update}($w_t$)
    \ENDIF
    \alglinelabel{alg:intra-critic-2}
    \IF{\texttt{is\_policy\_update\_time}$(t)$}
    \alglinelabel{alg:intra-policy-1}
    \STATE $\theta_{t+1} \gets$
    \texttt{perform\_policy\_update}($\theta_t$)
    \ENDIF
    \alglinelabel{alg:intra-policy-2}
    \STATE $t \gets t + 1$
    \UNTIL{\texttt{episode\_complete}$(t)$}
    \STATE \textbf{Return:} $t$, $\State_t$, $w_t$, $\theta_t$,
    episode length $T = t - \tau$.
  \end{algorithmic}
\end{algorithm}

\begin{algorithm}[H]
  \caption{Inter-Episode Training}
  \label{alg:enhance}
  \begin{algorithmic}[1]
    \STATE \textbf{Hyperparameters:} initial schedule
    $\relprob_0,\lambda_0$, threshold $\nu$, transition horizon
    $T_{\mathrm{tran}}$, and all hyperparameters of the backbone RL algorithm.
    \STATE \textbf{Initialize:} global time $t \gets 0$, schedule
    parameters $\relprob \gets \relprob_0$, $\lambda \gets \lambda_0$,
    initial critic weights $w_0$, initial learning-policy weights
    $\theta_0$, and $\State_0 \leftsquigarrow \transit_0(\bullet)$.
    \WHILE{training is not complete}
    \STATE Run one episode by calling the intra-episode rollout
    routine defined by \Cref{alg:enhance-intra-episode} (denoted
    \textsc{EpRollout} below):\\
    \alglinelabel{alg:rollout-call}
    $
    (t,\State_t,w_t,\theta_t,T) \gets
    \textsc{EpRollout}(t,\State_t,w_t,\theta_t,\baselinepolicy,\relprob,\lambda,\nu).
    $
    \STATE Reset the environment for the next episode:
    $\State_t \leftsquigarrow \transit_0(\bullet)$.
    \STATE Compute $\eta \gets
    \min\left(1,\frac{t-1}{T_{\mathrm{tran}}}\right)$ and
    $\chi_0 \gets \relprob_0\sum_{k = 0}^{T-1}\lambda_0^k$.
    \alglinelabel{alg:inter-episode-1}
    \alglinelabel{alg:chi0}
    \STATE Compute $\chi \gets \chi_0+\eta(T-\chi_0)$.
    \alglinelabel{alg:chi}
    \STATE Set $\relprob \gets \relprob_0+\eta(1-\relprob_0)$.
    \alglinelabel{alg:relprob-update}
    \STATE Set
    $\lambda \gets
    \underset{\lambda' \in [0,1]}{\mathrm{solve}}\!
    \left(\chi =
      \relprob\sum_{k = 0}^{T-1}(\lambda')^k
    \right)$.
    \alglinelabel{alg:inter-episode-2}
    \ENDWHILE
  \end{algorithmic}
\end{algorithm}

\begin{remark}
  \label{rem:value-function-critic-version}
  The listings in \Cref{alg:enhance-intra-episode,alg:enhance} are
  stated in the $Q$-function form because it is more general for the
  present purpose.
  Indeed, the analogous listing for a value-function critic is
  obtained by the literal replacement of every occurrence of
  $Q^{w_t}(\state,\action)$ in the listings with $v^{w_t}(\state)$.
\end{remark}

\subsection{Transition Scheduling Strategy}
\label{sec:transition-scheduling-strategy}

The transition scheduling strategy is implemented in
\Cref{alg:enhance}, lines \ref{alg:inter-episode-1}--\ref{alg:inter-episode-2}.

The hyperparameters $\relprob$ and $\lambda$ fundamentally represent
the degree of trust in the learning policy's capabilities.
When these values are relatively low, the mechanism exhibits limited
confidence in the learning policy and maintains strong reliance on
the baseline policy for guidance.
Conversely, as $\relprob$ and $\lambda$ approach unity, the system
demonstrates increasing trust in the learning policy, invoking the
baseline policy less frequently.
When both parameters reach their target values $\relprob = 1$ and
$\lambda = 1$, the probabilistic selection reduces to $U_t \leq 1$,
which is always satisfied, effectively recovering the pure backbone
RL algorithm with no baseline policy involvement.

The core idea underlying the entire scheme is to begin training with
relatively low values of $\relprob$ and $\lambda$, then
systematically increase them throughout the training process until
they reach unity.
This way agency is gradually transferred from the baseline to the
learning policy.
Initially, the algorithm exhibits strong reliance on the baseline
policy, but as training progresses, the learning policy enhances its
exploration capabilities and performance, warranting increased trust
and autonomy.


While one cannot directly control the exact number of learning policy
actions per episode, a lower bound on this quantity can be
systematically managed through the hyperparameter update strategy
outlined in lines \ref{alg:inter-episode-1}-\ref{alg:inter-episode-2}.
Consider the update strategy in more detail.

Let $T$ denote the episode length.
The probabilistic selection mechanism guarantees that in each
episode, the learning policy is selected at least $\sum_{t = \tau}^{T
+ \tau - 1} \mathbb{I}\{U_t \leq \relprob \lambda^{t - \tau}\}$ times.
Taking expectations, an adjustable lower bound is obtained: $
\mathbb{E}\left[\sum_{t = \tau}^{T + \tau - 1} \mathbb{I}\{\Action_t
= \Action^{\learningpolicy}_t\}\right] \geq \relprob \sum_{t =
0}^{T-1} \lambda^{t}.
$

The strategy behind these updates works by making this lower bound
grow linearly across training episodes:
\begin{enumerate}
  \item $\relprob$ is incrementally increased after each episode
    until it reaches unity (line \ref{alg:relprob-update}),
  \item the following equation is solved for $\lambda$ to equate the
    lower bound to the target value $\chi$ (lines
    \ref{alg:chi}--\ref{alg:inter-episode-2}): $\relprob \sum_{t =
    0}^{T-1} \lambda^{t} = \chi,$ where $\chi$ increases linearly
    from a small initial value to $T$ over the duration of training
    (lines \ref{alg:inter-episode-1}--\ref{alg:chi}).
\end{enumerate}

It is recommended to initialize $\relprob_0 \in [0.8, 1.0]$ and
choose $\lambda_0$ such that $\tfrac{\relprob_0\sum_{t = 0}^{T-1}
\lambda_0^{t}}{T} \approx 0.2$.
This choice ensures a strong reliance on the baseline policy at the
early stages of training, while preventing excessive dominance that
would otherwise cause the collected training data to consist almost
exclusively of transitions generated by the baseline policy.
An example of the resulting profile is illustrated in
\Cref{fig:schedule_parameters}: $\relprob$ increases linearly toward
one, whereas $\lambda$ follows the nonlinear profile induced by
solving the lower-bound equation in lines
\ref{alg:chi}--\ref{alg:inter-episode-2}.
\section{Theoretical Analysis}\label{sec:theoretical_analysis}

\subsection{Goal-Reaching Analysis}\label{sec:goal_reaching_analysis}

The analysis begins by clarifying the policy notation used throughout
this section.
For each training step $t$, let $\policy_t$ denote the non-stationary
executed policy induced by Algorithm~\ref{alg:enhance-intra-episode}.
That is, $\policy_t$ is the composite training-time policy that, at
each decision step, executes either the learning-policy action or the
baseline-policy action according to the arbitration rule of
Algorithm~\ref{alg:enhance-intra-episode}.

The notion of \emph{high goal-reaching rates} is interpreted as follows.
During the initial training phase, that is, for $t <
T_{\mathrm{tran}}$, the majority of episodes terminate by reaching
the goal set $\G$.

An important remark is required at this point.
Reaching the goal set $\G$ generally requires a sufficient number of
interaction steps.
Therefore, each episode must be long enough to allow the policy
$\policy_t$ to reach $\G$.
Consequently, any theoretical interpretation of goal-reaching
behavior necessarily presumes that the episode length is sufficiently
large relative to the time scale on which $\policy_t$ approaches the goal set.
For a practical heuristic for choosing such a horizon, see
\Cref{rem:practical-horizon-heuristic}.
This issue is addressed more explicitly in
\Cref{thm:uniform_goal_reaching}, which relies on a quantitative
characterization of how trajectories under the baseline policy
approach the goal set from arbitrary initial states.

\paragraph{Scope of the result}
To avoid unnecessary technical complications, a setting is considered
in which the policy $\policy_t$ generated by
Algorithm~\ref{alg:enhance-intra-episode} is allowed to take as many steps as
needed within a single episode to reach the goal set $\G$.
In \Cref{thm:goal-reaching}, a setting is considered, in which the
dynamics evolve within an episode of unbounded length; consequently,
termination or truncation does not occur once the policy $\policy_t$
reaches $\G$.
Equivalently, in the notation of \Cref{alg:enhance-intra-episode},
\texttt{episode\_complete} is taken to be identically false.
Thus, \Cref{thm:goal-reaching} below should not be read as any kind
of guarantee.
Instead, the theorem offers a theoretical \emph{interpretation}: it
identifies a mechanism by which the policy $\policy_t$ inherits the
asymptotic $\varepsilon$-improbable goal-reaching property of the
baseline $\baselinepolicy$, and thereby explains why, in practice,
high goal-reaching rates are observed for $t < T_{\mathrm{tran}}$.
This interpretive framing is maintained throughout the paper --- in
the statement and proof of the theorem, as well as in the surrounding
discussion.

\begin{theorem}
  \label{thm:goal-reaching}
  Consider an episode of unbounded length generated by $\policy_t$.
  Suppose that:
  \begin{enumerate}
    \item\label{item:goal-reaching-critic-consistency} The critic
      approximation $Q^w(\state,\action)$ is chosen consistently with
      \Cref{rem:upper-consistent-critic}.
    \item The decay factor $\lambda$ satisfies $\lambda < 1$ at the
      beginning of the episode (which holds for $t < T_{\mathrm{tran}}$).
    \item The baseline policy $\baselinepolicy$ satisfies the
      $\varepsilon$-improbable goal-reaching property.
  \end{enumerate}
  Then the policy $\policy_t$ generated by
  Algorithm~\ref{alg:enhance-intra-episode} also satisfies the
  $\varepsilon$-improbable goal-reaching property.
\end{theorem}

\paragraph{Proof idea}
The arbitration rule can select the learning policy only through two
mechanisms: the critic-improvement condition and the stochastic
relaxation condition.
The critic-improvement condition can fire only finitely many times,
because each such event increases the episode-local benchmark
$Q_t^\dagger$ by at least $\nu$, while the critic is uniformly bounded above.
The stochastic relaxation condition also fires only finitely many
times almost surely, since
$\sum_{t=0}^{\infty}\relprob\lambda^t<\infty$ when $\lambda<1$.
Hence, after a finite random time, the composite policy coincides
with the baseline policy.
Since the baseline satisfies the $\varepsilon$-improbable
goal-reaching property from arbitrary initial states, the same
property is inherited by $\policy_t$.

\begin{proof}
  See \ref{app:proof_basic_goal_reaching}.
\end{proof}

\begin{remark}
  The above result concerns the composite training-time policy $\policy_t$.
  The goal-reaching analysis is extended to the standalone
  neural-network policy $\learningpolicy$ after training in
  \ref{sec:transfer}, where \Cref{thm:transfer} provides an explicit
  bound on the degradation of the goal-reaching probability in terms
  of the expected distance between trajectories sampled by
  $\policy_t$ and $\learningpolicy$.
\end{remark}

\begin{remark}[Practical horizon heuristic]
  \label{rem:practical-horizon-heuristic}
  The proof also gives a simple way to estimate how much additional
  episode length is needed before the baseline-dominated behavior
  becomes visible.
  Even if the learning policy is poor, it is still activated by the
  stochastic relaxation branch in expectation at least
  \[
    \sum_{t=0}^{\infty}\relprob\lambda^t = \frac{\relprob}{1-\lambda}
  \]
  times.
  Thus, as a practical rule of thumb, the episode horizon should
  allow the baseline policy enough effective steps to reach the goal
  even after roughly this many learning-policy interventions.
  Equivalently, when the baseline needs about $T^{\mathrm{b}}$ steps
  to reach $\G$, one should choose the episode horizon noticeably
  larger than $T^{\mathrm{b}}+\relprob/(1-\lambda)$ during the early
  training phase.
  The next section provides a constructive uniform result that makes
  this type of horizon reasoning formal, but it requires additional
  assumptions; the estimate above is intended only as a quick heuristic.
\end{remark}

\begin{remark}[On critic consistency in practice]
  The critic-consistency requirement in
  item~\ref{item:goal-reaching-critic-consistency} of
  \Cref{thm:goal-reaching} is used to keep the theoretical interpretation clean.
  It should not be read as a fragile practical prerequisite for the
  goal-reaching phenomenon.
  In implementation, even if the critic is not explicitly constrained
  according to \Cref{rem:upper-consistent-critic}, critic estimates
  typically do not keep increasing indefinitely along an episode.
  Once the critic values saturate, or otherwise stop producing
  margin-improving updates, the critic-triggered preference for the
  learning policy is no longer activated at every step.
  The arbitration mechanism can then select the baseline policy again.
  When the baseline policy is capable of recovering the system, these
  interventions can correct poor learning-policy actions and drive
  the trajectory back toward the goal set.
  This is why goal-reaching behavior can still be observed
  empirically without explicitly enforcing the critic-consistency construction.
\end{remark}


\subsection{Uniform Goal-Reaching Property}
\label{sec:uniform_goal_reaching}

Control-theoretic analyses often pay particular attention to formal
uniform results.
To address this aspect, this section states a stronger uniform
version of the goal-reaching interpretation used in
\Cref{sec:theoretical_analysis}.
The result concerns reaching time, uniform overshoot boundedness, and
the distribution of the reaching time.
It is not used in the experiments and should be viewed as an
additional theoretical result included for completeness.
Nevertheless, the assumptions entering the result are also explained
from an implementation viewpoint, including how they can be made
feasible in practice and how they may be enforced when applying the theorem.

For the uniform result, the stochastic relaxation term in the
intra-episode mechanism of \Cref{alg:enhance-intra-episode} is
analyzed with an additional critic-monotonicity gate.
Concretely, the stochastic relaxation factor
$\lambda^{t-\tau}\relprob$ appearing in the algorithm listing is
replaced, in the analysis, by $\lambda^{t-\tau}\rho^{\mathrm{rel}}_t$, where
\begin{equation}
  \label{eq:relaxation_monotonicity_gate}
  \rho^{\mathrm{rel}}_t := \relprob \cdot
  \mathbb{I}\left\{Q^{w_t}(\State_t, \Action^{\learningpolicy}_t)
  \geq Q^{w_\tau}(\State_\tau, \Action^{\learningpolicy}_\tau)\right\}.
\end{equation}
Equivalently, the analyzed factor is $\lambda^{t-\tau}\relprob$ times
the indicator in \eqref{eq:relaxation_monotonicity_gate}.
Here, $\tau$ is the beginning of the current episode.
The coefficient $\rho^{\mathrm{rel}}_t$ coincides with $\relprob$
under normal conditions but becomes zero when the current episode
trajectory follows an unfavorable path \ie when the critic value
drops below its initial level.
The gate prevents the stochastic branch from activating the learning
policy in such cases and is used only to make the uniform argument explicit.

The exposition begins with a few auxiliary definitions.
Definition~\ref{dfn:uniform_eta_improbable} refines
Definition~\ref{dfn:goal_reaching_property} by requiring a
class-$\KL$ certificate (Definition~\ref{dfn:class_kl}) that bounds
the decay of the distance to the goal over time.
Definition~\ref{dfn:superlevel_set} fixes the terminology for superlevel sets.
Definition~\ref{dfn:function_with_bounded_superlevel_sets} then
introduces a concept that is uncommon in the literature---functions
with bounded superlevel sets.
Although this definition is not frequently encountered, it is in fact
closely related to the well-known concept of \emph{radial
unboundedness} (see Definition~\ref{dfn:radially_unbounded}).
Specifically, a function has bounded superlevel sets if and only if
its negative is radially unbounded, or if a simple logarithmic
transformation of it is radially unbounded (see
\Cref{prop:functions_with_bounded_superlevel_sets}).
\begin{dfn}[Class-$\KL$ function]
  \label{dfn:class_kl}
  A function $\beta:\R_{\ge 0}\times\R_{\ge 0}\to\R_{\ge 0}$ belongs
  to class $\KL$ if the following properties hold: for each fixed
  $t\ge 0$, the mapping $r\mapsto\beta(r,t)$ is continuous, satisfies
  $\beta(0,t)=0$, and is strictly increasing on $\R_{\ge 0}$; and for
  each fixed $r\ge 0$, the mapping $t\mapsto\beta(r,t)$ is
  nonincreasing and satisfies $\lim_{t\to\infty}\beta(r,t)=0$.
\end{dfn}

\begin{dfn}
  \label{dfn:uniform_eta_improbable}
  A policy $\pi \in \policies_{\text{nstat}}$ is said to satisfy the
  \emph{uniform $\varepsilon$-improbable goal-reaching property} if
  there exists a function $\beta \in \KL$ such that, for all initial
  states $\state_0 \in \states$, $$
  \PP{\goaldist(\State^{\pi}_t(\state_0)) \le
  \beta(\goaldist(\state_0), t) \text{ for all } t \ge 0} \ge 1 - \varepsilon.
  $$
\end{dfn}

\begin{dfn}[Superlevel set]
  \label{dfn:superlevel_set}
  Let $X$ be a set, let $f:X\to\R$, and let $a\in\R$.
  The $a$-\emph{superlevel set} of $f$ on $X$ is the set
  \[
    \{x\in X \mid f(x)\ge a\}.
  \]
\end{dfn}

\begin{dfn}
  \label{dfn:function_with_bounded_superlevel_sets}
  Let $X$ be a metric space and $f : X \to \R$.
  The function $f$ is said to have bounded superlevel sets if, for
  every $a \in f (X) = \{ f (x) \mid x \in X\}$, the $a$-superlevel
  set of $f$ on $X$ is bounded in $X$.
\end{dfn}

\begin{dfn}
  \label{dfn:radially_unbounded}
  A function $f:\R^n \to \R$ is \emph{radially unbounded} if
  $\lim_{\|x\|\to +\infty} f(x) = +\infty$.
\end{dfn}

\begin{proposition}[Equivalent characterizations of functions with
  bounded superlevel sets]
  \label{prop:functions_with_bounded_superlevel_sets}
  Let $f:\R^n \to \R$.
  The following statements are equivalent:
  \begin{enumerate}[label=(\roman*)]
    \item\label{item:bounded_superlevel_sets} $f(x)$ has bounded
      superlevel sets;
    \item\label{item:liminf} $\displaystyle \lim_{\lVert x \rVert \to
      \infty} f(x) = \inf_{x' \in \R^n} f(x') \in [-\infty,
      +\infty)$, and $\displaystyle f(x) > \inf_{x' \in \R^n} f(x')$
      for all $x \in \R^n$;
    \item\label{item:radially_unbounded} $-f(x)$ is radially
      unbounded, or there exists a finite $\inf_{x \in \R^n} f(x) =:
      f^{\inf} \in \R$ such that $-\log\!\big(f(x) - f^{\inf}\big)$
      is radially unbounded.
  \end{enumerate}
\end{proposition}

\begin{proof}
  See \ref{app:proof_bounded_superlevel_sets}.
\end{proof}

There are two technical assumptions needed to prove
\Cref{thm:uniform_goal_reaching}: Assumption~\ref{ass:valuebase} and
Assumption~\ref{ass:uniform_fallback}.
Assumption~\ref{ass:uniform_fallback} is simply a requirement that
the baseline policy satisfy Definition~\ref{dfn:uniform_eta_improbable}.
Assumption~\ref{ass:valuebase} requires additional discussion.

In particular, it is necessary to ensure that the critic function
$Q^w(\state, \action)$ remains uniformly bounded between two
continuous envelope functions for all states, actions, and critic
parameters: $$ \kappalow(\state) \le Q^w(\state, \action) \le
\kappahigh(\state), \quad \forall \state\in\states,
\action\in\actions, w\in\mathbb{W}.
$$ There are several practical ways to guarantee this property.
One option is to explicitly \emph{clip} the critic outputs, $$
Q^w(\state, \action) \leftarrow \operatorname{clip}\left(Q^w(\state,
\action), \kappalow(\state), \kappahigh(\state)\right), $$ which
enforces the desired bounds by construction.
Equivalently, the last critic layer can be parameterized through any
suitable bounded activation and then linearly transformed to the
interval $[\kappalow(\state),\kappahigh(\state)]$; examples include a
linearly transformed $\tanh$ activation or a bounded folding map such
as $x\mapsto \arccos(\cos x)$.
Alternatively, one can \emph{regularize the critic parameters} $w$
(for instance, using spectral normalization or weight decay) so that
the critic values cannot diverge outside the admissible range.
Both approaches make Assumption~\ref{ass:valuebase} easy to satisfy
in practice while preserving continuity and stability of the critic,
ensuring the boundedness conditions required for the proof below.

As for the practical choice of $\kappalow$ and $\kappahigh$,
\Cref{rem:upper-consistent-critic} provides the natural upper scale
$\bar Q := \sum_{k=0}^{\infty}\gamma^k r^{\max}$.
Thus, the upper envelope can be chosen as $\kappahigh(\state) \equiv \bar Q$.
If the \emph{rewards are bounded on both sides}, the lower envelope
can be chosen analogously.
When $r^{\min} \le r(\state, \action) \le r^{\max}$, one can simply
define $\kappalow(\state) = \sum_{k=0}^{\infty}\gamma^k r^{\min} -
\|\state - \state_{\G}\|^2$ and $\kappahigh(\state) \equiv \bar Q$,
where $\state_{\G}$ is the center of the goal set $\G$.
In this common case, Assumption~\ref{ass:valuebase} imposes no
additional restriction: it merely formalizes the natural boundedness
of the critic that follows from bounded rewards and discounting.

If the \emph{rewards are unbounded from below} (for example, for
negative quadratic costs), one can estimate empirical bounds by
sampling several random trajectories under the initial policy,
evaluating \(Q^{w_0}(\state_t, \action_t)\), and then defining
\(\kappalow\) using the empirical minimum.
These initial estimates can then be refined adaptively as training progresses.
Such constructions ensure that the critic operates within a
numerically stable and well-defined range, making
Assumption~\ref{ass:valuebase} both theoretically sound and
straightforward to implement in practice.

\begin{remark}[On critic envelope constraints]
  Assumption~\ref{ass:valuebase} should be understood as a design
  constraint on the critic class, not as a claim that an unconstrained
  neural-network critic automatically satisfies such bounds during
  training. The motivation for using state-dependent comparison
  envelopes is standard in Lyapunov analysis: continuous positive
  definite functions can be bounded from below and above by
  class-$\mathcal{K}$ functions, and if they are radially unbounded
  the bounds can be chosen from class-$\mathcal{K}_{\infty}$
  \cite[Lemma~4.3]{Khalil2002NonlinearSystems}.
  Thus, for an ideal value- or Lyapunov-like object, a role that the
  optimal critic can play under the usual control interpretation, the
  existence of comparison-function envelopes is a classical fact.
  In the present algorithm this fact is used as motivation for
  restricting the critic approximator by construction. For example,
  with rewards bounded below by \(r^{\min}\), one may use the same
  envelopes as above,
  \(\kappalow(\state)=\sum_{k=0}^{\infty}\gamma^k r^{\min}
  -\|\state-\state_{\G}\|^2\) and
  \(\kappahigh(\state)\equiv\bar Q\), and enforce
  \(\kappalow(\state)\le Q^w(\state,\action)\le
  \kappahigh(\state)\) through output clipping or a bounded final
  layer. The bounded-superlevel-set property required of
  \(\kappalow\) then follows from the quadratic term in
  \(\|\state-\state_{\G}\|^2\). Consequently,
  Assumption~\ref{ass:valuebase} is a feasible critic-design
  requirement rather than an automatic property of an arbitrary trained
  neural critic.
\end{remark}

\begin{theorem}
  \label{thm:uniform_goal_reaching}
  Consider the intra-episode process generated by
  \Cref{alg:enhance-intra-episode}, with the stochastic relaxation
  term interpreted through the critic-monotonicity gate
  $\rho^{\mathrm{rel}}_t$ in \eqref{eq:relaxation_monotonicity_gate}.
  Without loss of generality set $\tau = 0$.
  Let this process be initialized at $\state_0$ with
  $\goaldist(\state_0) \le d^{\circ}$, where $d^{\circ} \in \R_{>0}$
  is arbitrary.

  \noindent Assume that:
  \begin{enumerate}[label=(\subscript{\mathrm{A}}{{\arabic*}}),
      leftmargin=2.5em]
    \item\label{ass:valuebase} The function $Q^w(\state, \action)$
      admits lower and upper bounding continuous functions
      $\kappalow(\state)$ and $\kappahigh(\state)$, respectively,
      where $\kappalow(\state)$ also has bounded superlevel sets:
      \[ \kappalow(\state) \leq Q^{w}(\state,
        \action) \leq \kappahigh(\state).
      \]
      for all $\state \in \states$, $\action \in \actions$ and $w \in
      \mathbb{W}$.

    \item\label{ass:uniform_fallback} $\baselinepolicy$ satisfies the
      \emph{uniform $\eps$-improbable goal-reaching property} with
      certificate $\beta \in \KL$ (see Definition
      \ref{dfn:uniform_eta_improbable}).
  \end{enumerate}

  \noindent Then the following claims hold:
  \begin{enumerate}[label=(\subscript{\mathrm{C}}{{\arabic*}}),
      leftmargin=2.5em]

    \item\label{claim:uniform_overshoot_bound}
      \textit{(\(\eps\)-improbable uniform overshoot boundedness)}
      There exists $\delta(d^{\circ})\in\R_{>0}$ such that
      \[
        \PP{\goaldist\left(\State_t^{\policy_t}(\state_0)\right) \le
        \delta(d^{\circ})\, \text{ for all } t \ge 0} \ge 1-\eps.
      \]

    \item\label{claim:uniform_reaching_time}
      \textit{(\(\eps\)-improbable uniform reaching time)} For each
      $d^{*} \in (0, d^{\circ})$, there is an almost surely finite
      random time $T(d^{\circ},d^{*})\in\R_{\geq0}$ such that
      \[
        \PP{\goaldist\left(\State_t^{\policy_t}(\state_0)\right) \le
        d^{*} \text{ for all } t \ge T(d^{\circ},d^{*})} \ge\ 1-\eps.
      \]

    \item\label{claim:distribution_reaching_time}
      \textit{(Reaching time distribution)} There exist natural
      numbers $\tau(d^{\circ})$ and
      $\tau^{\mathrm{b}}(d^{\circ},d^{*})$ such that for all $t \in
      \mathbb{Z}_{\ge 0}$,
      \begin{equation*}
        \PP{T(d^{\circ},d^{*}) \le (\tau(d^{\circ}) + t)\,
        \tau^{\mathrm{b}}(d^{\circ},d^{*})}
        =
        \prod_{k=t}^{\infty}\left(1 - \lambda^k \relprob\right).
      \end{equation*}
      where $\prod_{k=t}^{\infty}\left(1 - \lambda^k \relprob\right)
      \ra 1$ as $t \ra \infty$.
  \end{enumerate}
  Furthermore, $\tau(d^{\circ})$,
  $\tau^{\mathrm{b}}(d^{\circ},d^{*})$, and $\delta(d^{\circ})$ are
  given explicitly in \eqref{eq:uniform_tau},
  \eqref{eq:uniform_tau_b}, and \eqref{eq:uniform_delta}, respectively.
\end{theorem}

\paragraph{Interpretation}
\Cref{claim:uniform_overshoot_bound} is the bounded-excursion claim.
Starting from any state with initial distance at most $d^{\circ}$,
the closed-loop trajectory remains inside a finite operational tube
of radius $\delta(d^{\circ})$ with probability at least $1-\eps$.
This claim does not say that the trajectory is already near the goal.
It says that the learning-policy insertions cannot make the process
escape arbitrarily far before recovery by the baseline policy.

\Cref{claim:uniform_reaching_time} is the eventual-settling claim.
For any requested final tolerance $d^{*}<d^{\circ}$, there is a
finite random time $T(d^{\circ},d^{*})$ after which the trajectory
stays inside the $d^{*}$-neighborhood of the goal, again with
probability at least $1-\eps$.
Thus, temporary deviations are allowed before $T(d^{\circ},d^{*})$,
but after that time the closed-loop behavior is interpreted as
settled near the goal.

\Cref{claim:distribution_reaching_time} explains how large this
settling time can be.
It states that the time needed to reach and remain inside the
$d^{*}$-neighborhood can be computed constructively from the
quantities appearing in the theorem.
In practice, this reaching time is a random variable because the
relaxation mechanism is stochastic.
However, its distribution is explicit: $\PP{T(d^{\circ},d^{*}) \le
(\tau(d^{\circ})+t)\tau^{\mathrm{b}}(d^{\circ},d^{*})} =
\prod_{k=t}^{\infty}(1-\lambda^k\relprob)$.
Thus, for any selected probability level, this formula specifies a
time by which the closed-loop process has reached the
$d^{*}$-neighborhood of the goal set.

\paragraph{Proof idea}
The proof is constructive.
First define
\begin{align}
  v^{\min}(d^{\circ})
  &:= \min\{\kappalow(\state):
  \goaldist(\state)\le d^{\circ}\},
  \label{eq:uniform_vmin}\\
  \mathbb{V}(d^{\circ})
  &:= \{\state:\kappalow(\state)\ge v^{\min}(d^{\circ})\},
  \label{eq:uniform_superlevel_set}\\
  v^{\max}(d^{\circ})
  &:= \max\{\kappahigh(\state):
  \state\in\mathbb{V}(d^{\circ})\},
  \label{eq:uniform_vmax}\\
  \tau(d^{\circ})
  &:= 1+\left\lfloor
  \frac{v^{\max}(d^{\circ})-v^{\min}(d^{\circ})}{\nu}
  \right\rfloor .
  \label{eq:uniform_tau}
\end{align}
The gated relaxation ensures that the learning policy is invoked only
inside $\mathbb{V}(d^{\circ})$.
Consequently, the critic-improvement branch can fire at most
$\tau(d^{\circ})$ times.
Next set
\begin{align}
  d^{\bar{\mathrm{p}}}(d^{\circ})
  &:= \sup\{\bar{\transit}(\state,\action):
  \state\in\mathbb{V}(d^{\circ}),\,\action\in\actions\},
  \label{eq:uniform_transition_bound}\\
  d^{\max}(d^{\circ})
  &:= \max\{d^{\circ},d^{\bar{\mathrm{p}}}(d^{\circ})\},
  \label{eq:uniform_dmax}\\
  \delta(d^{\circ})
  &:= \beta(d^{\max}(d^{\circ}),0).
  \label{eq:uniform_delta}
\end{align}
Finally, using a standard decomposition
$\beta(d,t)\le\kappa(d)\xi(e^{-t})$ with $\kappa,\xi\in\K_{\infty}$, define
\begin{equation}
  \label{eq:uniform_tau_b}
  \tau^{\mathrm{b}}(d^{\circ},d^{*})
  :=
  \max\left\{
    1,\,
    \left\lceil
    -\log\left(\xi^{-1}\left(
        \frac{d^{*}}{\kappa(d^{\max}(d^{\circ}))}
    \right)\right)\right\rceil
  \right\}.
\end{equation}
The random relaxation activations are controlled by
\begin{align}
  T^{\mathrm{rel}}
  &:= \inf\{t\ge0:U_k\ge\lambda^k\relprob
  \text{ for all }k\ge t\},
  \label{eq:uniform_relaxation_time}\\
  \PP{T^{\mathrm{rel}}\le t}
  &= \prod_{k=t}^{\infty}(1-\lambda^k\relprob).
  \label{eq:uniform_relaxation_distribution}
\end{align}
Thus the learning policy can interrupt the baseline no more than
$\tau(d^{\circ})+T^{\mathrm{rel}}$ times, and each interruption
returns the state to a region from which
$\tau^{\mathrm{b}}(d^{\circ},d^{*})$ baseline steps suffice, with
probability at least $1-\eps$, to remain within $d^{*}$ of the goal.
This gives
\begin{equation}
  \label{eq:uniform_reaching_time_formula}
  T(d^{\circ},d^{*})
  =
  \bigl(\tau(d^{\circ})+T^{\mathrm{rel}}\bigr)
  \tau^{\mathrm{b}}(d^{\circ},d^{*}),
\end{equation}
from which the overshoot, reaching-time, and distributional claims follow.

\begin{proof}
  See \ref{app:proof_uniform_goal_reaching}.
\end{proof}

\paragraph{Extension to value functions}
The preceding analysis applies equally when the critic is a
state--value function.
To make this explicit, replace $Q^{w}(\state,\action)$ by
$\Value^{w}(\state)$ throughout the gated intra-episode process
considered in \Cref{thm:uniform_goal_reaching} and track the
reference level $\Value^{\dagger}_{t}$ instead of $Q^{\dagger}_{t}$.
Equivalently, this is the special case in which the critic is
independent of the action argument.
The structural assumptions translate directly
(Assumption~\ref{ass:valuebase} $\to$
  Assumption~\ref{ass:valuebase_value},
  Assumption~\ref{ass:uniform_fallback} $\to$
Assumption~\ref{ass:uniform_fallback_value}), either by viewing
$\Value^{w}$ as an aggregation of $Q^{w}$ or by postulating envelopes
for $\Value^{w}$ itself.
With this substitution, the definitions in the proof of
\Cref{thm:uniform_goal_reaching} remain unchanged, so the overshoot
bound, the uniform reaching time, and the distributional statement
follow by the same argument.
These facts are collected in \Cref{thm:uniform_goal_reaching_value} below.

\begin{corollary}
  \label{thm:uniform_goal_reaching_value}
  Consider the value-critic analogue of the gated intra-episode
  process in \Cref{thm:uniform_goal_reaching}, initialized at
  $\state_0$ with $\goaldist(\state_0) \le d^{\circ}$, where
  $d^{\circ} \in \R_{>0}$ is arbitrary.

  \noindent Assume that:
  \begin{enumerate}[label=(\subscript{\mathrm{A}^{v}}{{\arabic*}}),
      leftmargin=2.5em]
    \item\label{ass:valuebase_value} The function
      $\Value^{w}(\state)$ admits lower and upper bounding continuous
      functions $\kappalow(\state)$ and $\kappahigh(\state)$,
      respectively, where $\kappalow(\state)$ also has bounded
      superlevel sets: $$
      \kappalow(\state) \leq \Value^{w}(\state) \leq \kappahigh(\state).
      $$ for all $\state \in \states$ and $w \in \mathbb{W}$.

    \item\label{ass:uniform_fallback_value} $\baselinepolicy$
      satisfies the \emph{uniform $\eps$-improbable goal-reaching
      property} with certificate $\beta \in \KL$ (see Definition
      \ref{dfn:uniform_eta_improbable}).
  \end{enumerate}

  \noindent Then the following claims hold:
  \begin{enumerate}[label=(\subscript{\mathrm{C}}{{\arabic*}}),
      leftmargin=2.5em]

    \item\label{claim:uniform_overshoot_bound_value}
      \textit{(\(\eps\)-improbable uniform overshoot boundedness)}
      There exists $\delta(d^{\circ})\in\R_{>0}$ such that
      \[
        \PP{\goaldist\left(\State_t^{\policy_t}(\state_0)\right) \le
        \delta(d^{\circ})\, \text{ for all } t \ge 0} \ge 1-\eps.
      \]

    \item\label{claim:uniform_reaching_time_value}
      \textit{(\(\eps\)-improbable uniform reaching time)} For each
      $d^{*} \in (0, d^{\circ})$, there is an almost surely finite
      random time $T(d^{\circ},d^{*})\in\R_{\geq0}$ such that
      \[
        \PP{\goaldist\left(\State_t^{\policy_t}(\state_0)\right) \le
        d^{*} \text{ for all } t \ge T(d^{\circ},d^{*})} \ge\ 1-\eps.
      \]

    \item\label{claim:distribution_reaching_time_value}
      \textit{(Reaching time distribution)} There exist natural
      numbers $\tau(d^{\circ})$ and
      $\tau^{\mathrm{b}}(d^{\circ},d^{*})$ such that for all $t \in
      \mathbb{Z}_{\ge 0}$,
      \begin{equation*}
        \PP{T(d^{\circ},d^{*}) \le (\tau(d^{\circ}) + t)\,
        \tau^{\mathrm{b}}(d^{\circ},d^{*})}
        =
        \prod_{k=t}^{\infty}\left(1 - \lambda^k \relprob\right).
      \end{equation*}
      where $\prod_{k=t}^{\infty}\left(1 - \lambda^k \relprob\right)
      \ra 1$ as $t \ra \infty$.
  \end{enumerate}
  Furthermore, $\tau(d^{\circ})$,
  $\tau^{\mathrm{b}}(d^{\circ},d^{*})$, and $\delta(d^{\circ})$ are
  given explicitly by the same formulas \eqref{eq:uniform_tau},
  \eqref{eq:uniform_tau_b}, and \eqref{eq:uniform_delta}, respectively.
\end{corollary}

\paragraph{Proof idea}
The proof repeats the constructive argument for
\Cref{thm:uniform_goal_reaching} after the literal substitutions
$Q^{w}(\state,\action)\mapsto\Value^{w}(\state)$ and
$Q^{\dagger}_{t}\mapsto\Value^{\dagger}_{t}$.
The gated relaxation coefficient is replaced by
\[
  \relprob\, \mathbb{I}\{\Value^{w_t}(\State_t)\ge \Value^{w_0}(\State_0)\},
\]
and the same formulas for $\tau(d^{\circ})$,
$\tau^{\mathrm{b}}(d^{\circ},d^{*})$, $\delta(d^{\circ})$, and
$T(d^{\circ},d^{*})$ in
\eqref{eq:uniform_vmin}--\eqref{eq:uniform_reaching_time_formula} apply.

\begin{proof}
  See \ref{app:proof_uniform_value_corollary}.
\end{proof}

\subsection{Goal-Reaching Transfer to the Neural Policy}
\label{sec:transfer}

A key feature of the policy produced by~\Cref{alg:enhance} is its
reliance on the baseline fallback mechanism during training,
particularly in the initial phase where $\lambda < 1$.
As established in
\Cref{thm:goal-reaching,thm:uniform_goal_reaching,thm:uniform_goal_reaching_value},
this mechanism admits a theoretical interpretation that explains the
high goal-reaching rates observed in practice, consistent with the
experimental findings reported in~\Cref{sec:experiments}.
A natural question arises: once the learned neural-network
policy~$\learningpolicy$ is deployed without baseline
intervention---that is, with the schedule parameters set to their
terminal values ($\lambda = 1$, $\relprob = 1$), so that every action
is sampled directly from~$\learningpolicy(\bullet \mid
\State_t)$---what can be said about the ability to reach the goal set~$\G$?
This section provides a formal answer to this question.

\subsubsection{Deployment setting and trajectory distance}
\label{sec:transfer_setup}

Suppose that at training step~$t$ the learning policy has
weights~$\theta_t = \theta$.
Both deployment regimes below are obtained from
\Cref{alg:enhance-intra-episode} by initializing a new episode at
deployment time and using the rollout routine only for action
selection: learning is disabled, so \texttt{is\_critic\_update\_time}
and \texttt{is\_policy\_update\_time} always return \texttt{False}.
Thus the critic weights~$w$, learning-policy weights~$\theta$, and
the schedule parameters used in action selection remain fixed
throughout deployment.
\begin{enumerate}[label=(\subscript{\mathrm{R}}{{\arabic*}}),
    ref=\ensuremath{\mathrm{R}_{\arabic*}}, leftmargin=2.5em]
  \item\label{regime:fallback} \textbf{Policy~$\policy_t$ with
    baseline fallback} --- the action-selection mechanism of
    \Cref{alg:enhance} is retained with the current training-time
    schedule values~$\lambda < 1$ and~$\relprob \le 1$.
    The baseline policy~$\baselinepolicy$ may therefore still be
    invoked as a fallback whenever the critic condition is not met.
  \item\label{regime:neural} \textbf{Pure neural-network mode} ---
    all schedule parameters are set to their terminal values
    ($\lambda = 1$, $\relprob = 1$), so every action is sampled
    directly from~$\learningpolicy(\bullet
    \mid \State_t)$ without any baseline fallback.
\end{enumerate}

The two deployment regimes are formalized as the operation of two
distinct policies.
\ref{regime:fallback} corresponds to the policy~$\policy_t$ generated
by the frozen fallback instance of \Cref{alg:enhance}, while
\ref{regime:neural} corresponds to the pure neural-network instance
of the same algorithm, denoted by~$\learningpolicy$.
Both policies are initialized at a common state~$\state_0 \in
\states$ and produce respective random state sequences
$(\State_0^{\policy_t}, \State_1^{\policy_t}, \ldots) \in \states^T$
and $(\State_0^{\learningpolicy}, \State_1^{\learningpolicy}, \ldots)
\in \states^T$, with $\State_0^{\learningpolicy} =
\State_0^{\policy_t} = \state_0$, where $T \in \N \cup \{\infty\}$ is
the episode length.

\subsubsection{Expected trajectory distance}
\label{sec:trajectory_distance}

Hereafter, the Python-like shorthand~$X_{0:T}$ denotes the sequence
$(X_0, X_1, \dotsc, X_{T-1})$ (or $(X_0, X_1, \dotsc)$ when $T =
\infty$), so that, for instance, the expressions above contract
to~$\State_{0:T}^{\policy_t}$ and~$\State_{0:T}^{\learningpolicy}$.

The strategy for answering the central question---what can be said
about the ability of the learned policy~$\learningpolicy$ to reach
the goal set~$\G$ when deployed on its own---is straightforward: if
the trajectories $\State_{0:T}^{\learningpolicy}$
and~$\State_{0:T}^{\policy_t}$ are close to each other, then one can
estimate how close $\State_{0:T}^{\learningpolicy}$ is to the goal set.
The first step is therefore to quantify the distance between the
trajectories $\State_{0:T}^{\learningpolicy}$ and $\State_{0:T}^{\policy_t}$.
A natural single-rollout distance is
$d_{\states^T}\left(\State_{0:T}^{\learningpolicy},
\State_{0:T}^{\policy_t}\right) := \sup_{0 \le t < T}
\|\State_t^{\learningpolicy} - \State_t^{\policy_t}\|$, the
worst-case state deviation over the episode.
Since both state sequences are random processes, this supremum is
itself a random variable.
To obtain a scalar distance, the two processes are realized on a
fixed common probability space and the expected sup-distance is used.
For simplicity, this paper uses the independent-rollout convention:
conditional on the common initial state, the rollout under
$\learningpolicy$ and the rollout under $\policy_t$ are generated independently.

\begin{dfn}[Independent-rollout trajectory distance]
  \label{dfn:trajectory-distance}
  Let $\policy^1, \policy^2$ be two (generally non-stationary)
  policies acting from a common initial state $\state_0 \in \states$
  under the transition kernel~$\transit$.
  Fix a horizon $T \in \N \cup \{\infty\}$.
  For each $k \in \{1, 2\}$, denote by $\State_{0:T}^{\policy^k}$ the
  random trajectory generated by~$\policy^k$.
  The two trajectories are sampled independently.
  The trajectory distance is
  \begin{equation}\label{eq:trajectory-distance}
    \DT(\policy^1, \policy^2)
    :=
    \E{\sup_{0 \le t < T}
    \|\State_t^{\policy^1} - \State_t^{\policy^2}\|}.
  \end{equation}
\end{dfn}

The independence convention is not essential for \Cref{thm:transfer}.
The proof of \Cref{thm:transfer} only uses a fixed joint realization
of the two state sequences and an upper bound on the corresponding
expected sup-distance.
Thus, the same theorem would remain valid under a shared-noise
realization, an independent-noise realization, or any other
prescribed joint construction, provided $\DT$ is replaced by the
expected trajectory distance under that construction.
The independence convention is used only to make the quantity in
\eqref{eq:trajectory-distance} unambiguous and easy to estimate.

A concrete and readily implementable estimate can be obtained as follows.
Consider $N$ independent pairs of trajectories, indexed by $n = 1, \ldots, N$.
Within each pair, one trajectory is generated by the learning
policy~$\learningpolicy$ (with $\lambda = 1$, $\relprob = 1$) and the
other by the deployment policy~$\policy_t$ from \ref{regime:fallback}
(with the actual schedule parameters), using independent simulator
randomness after the shared initial state.
Denote the resulting state sequences
by~$\State_{0:T}^{\learningpolicy,\, n}$
and~$\State_{0:T}^{\policy_t,\, n}$, respectively.
The expected trajectory distance in~\eqref{eq:trajectory-distance}
can then be estimated by the Monte Carlo average
\begin{equation}\label{eq:mc_distance_estimate}
  \widehat{D}^T_N
  \;:=\;
  \frac{1}{N} \sum_{n=1}^{N}
  \sup_{0 \le t < T}
  \left\|\State_t^{\learningpolicy,\, n}
  - \State_t^{\policy_t,\, n}\right\|,
\end{equation}
which is a consistent estimator of $\DT(\learningpolicy, \policy_t)$
under the independent-rollout convention.

\begin{figure}[h]
  \centering
  \includegraphics[width=0.6\linewidth]{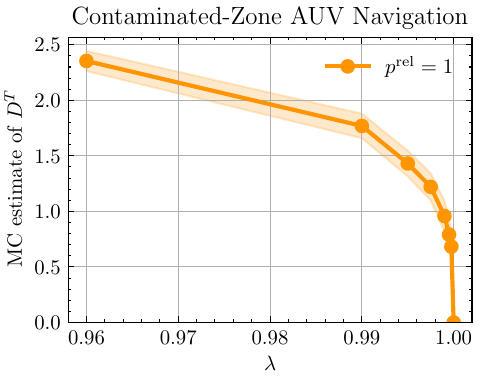}
  \caption{
    Monte Carlo estimate of the trajectory distance in
    \eqref{eq:mc_distance_estimate} for the Contaminated-Zone AUV
    Navigation environment.
    The estimate is computed from independent rollout pairs of the
    saved learning-policy checkpoint at the baseline-removal point
    and the corresponding policy with baseline fallback, using
    $p^{\mathrm{rel}}=1$ and varying the within-episode decay factor $\lambda$.
  }
  \label{fig:drone_checkpoint_lambda_distance}
\end{figure}

\subsubsection{Goal-reaching transfer theorem}
\label{sec:transfer_theorem}

\begin{theorem}[Goal-reaching transfer via trajectory distance]
  \label{thm:transfer}
  Fix a horizon $T \in \N \cup \{\infty\}$ and an initial state
  $\state_0 \in \states$.
  Consider regimes~\ref{regime:fallback} and~\ref{regime:neural},
  both executed from~$\state_0$.
  Define the \emph{settling time}
  \begin{equation}\label{eq:settling}
    \tau_T(d,\, \policy)
    \;:=\;
    \min\left\{0 \leq t' < T:
      \sup_{t' \leq k < T}
    \goaldist\left(\State_k^{\policy}\right) \le d\right\}
  \end{equation}
  (set $\tau_T = \infty$ if no such $t'$ exists).
  \noindent Assume that:
  \begin{enumerate}[label=(\subscript{\mathrm{A}}{{\arabic*}}),
      leftmargin=2.5em]
    \item\label{ass:transfer_settling} $\PP{\tau_T(d^*,\, \policy_t) < T}
      \ge 1 - \eps$ \ie the deployment policy~$\policy_t$ produced by
      regime~\ref{regime:fallback} settles within the
      $d^*$-neighborhood of~$\G$ and remains there until the end of
      the episode with probability at least $1 - \eps$;

    \item\label{ass:transfer_distance} $\DT(\learningpolicy, \policy_t)
      \le \Delta_T$ for some $\Delta_T \ge 0$.
  \end{enumerate}

  \noindent Then the following claim holds:
  \begin{enumerate}[label=(\subscript{\mathrm{C}}{{\arabic*}}),
      leftmargin=2.5em]
    \item\label{claim:transfer}
      \textit{(Settling transfer to regime~\ref{regime:neural})} For
      every $\delta > 0$:
      \begin{equation}\label{eq:transfer_bound}
        \PP{\tau_T(d^* + \delta,\, \learningpolicy) < T}
        \;\ge\;
        1 - \eps - \frac{\Delta_T}{\delta}\,.
      \end{equation}
  \end{enumerate}
\end{theorem}

\paragraph{Proof idea}
Realize the trajectory of the fallback deployment policy~$\policy_t$
and the trajectory of the pure neural-network
policy~$\learningpolicy$ independently, as in
Definition~\ref{dfn:trajectory-distance}.
Their expected sup-norm distance is then controlled by
$\DT(\learningpolicy,\policy_t)$.
On the event that the two trajectories remain within distance
$\delta$ of each other over the horizon~$T$, the triangle inequality
implies that every trajectory of~$\policy_t$ that settles inside the
$d^*$-neighborhood of~$\G$ yields a trajectory of~$\learningpolicy$
that settles inside the $(d^*+\delta)$-neighborhood.
The only loss comes from the event that the two jointly realized
trajectories separate by more than~$\delta$; Markov's inequality
bounds this probability by $\DT(\learningpolicy,\policy_t)/\delta$,
and hence by $\Delta_T/\delta$.
Combining this loss with the assumed settling probability $1-\eps$
for~$\policy_t$ gives~\eqref{eq:transfer_bound}.

\begin{proof}
  See \ref{app:proof_transfer}.
\end{proof}

\begin{remark}[Optimal slack]
  \label{rem:optimal_delta}
  The bound~\eqref{eq:transfer_bound} involves a trade-off between
  the neighborhood size $d^* + \delta$ and the probability loss
  $\Delta_T/\delta$.
  Setting $\delta = \sqrt{\Delta_T}$ yields
  \[
    \PP{\tau_T(d^* + \sqrt{\Delta_T},\, \learningpolicy) < T} \;\ge\;
    1 - \eps - \sqrt{\Delta_T}\,,
  \]
  which converges to $1 - \eps$ as $\Delta_T \to 0$.
\end{remark}

\section{Experiments}
\label{sec:experiments}

The proposed approach is validated in two environments:
\begin{enumerate}
  \item \textbf{Contaminated-Zone Autonomous Underwater Vehicle (AUV)
    Navigation} (\Cref{sec:underwaterdrone}): the agent must reach a
    target location while avoiding contaminated areas.
  \item \textbf{Treasure-Collecting Robot} (\Cref{sec:kin-robot}):
    the agent must reach a target region while collecting a
    high-reward treasure along the way.
\end{enumerate}
The choice of environments is deliberate.
To demonstrate the key advantages of the proposed method, the
evaluation task should have a goal set with a clear operational
meaning and a baseline policy that can be specified explicitly.
Many standard Gymnasium benchmarks are less suitable for this
purpose: in complex locomotion tasks, such as humanoid control,
constructing an explicit baseline policy is itself a difficult
problem; in simpler benchmarks, such as Pendulum, Inverted Pendulum,
or Mountain Car, the goal-reaching structure is too elementary to
stress the proposed mechanism; and in many generic benchmarks the
goal set is not as explicitly interpretable.
These two environments were chosen to expose the intended setting:
goal regions are explicit, baseline policies are available in closed
form, and the reward structure still leaves room for nontrivial
learning behavior.
They are also specified transparently: the dynamics, reward
functions, goal regions, and baseline policies are given explicitly.
This makes it possible to introduce demonstration metrics, such as
goal-reaching and constraint-avoidance rates, that directly reflect
the behavior targeted by the method.
Such metrics are complementary to cumulative reward, whose numerical
value is often difficult to interpret on its own.

The section first describes the two evaluation environments and their
associated tasks.
It then reports the performance of the standard TD3 and SAC methods,
along with their residual RL variants and the proposed approach
instantiated on top of the same TD3 and SAC backbones, as summarized
in \Cref{sec:exp_results}.

\subsection{Contaminated-Zone AUV Navigation}\label{sec:underwaterdrone}

\begin{figure}[!b]
  \centering
  \begin{minipage}{0.48\linewidth}
    \centering
    \includegraphics[width=\linewidth]{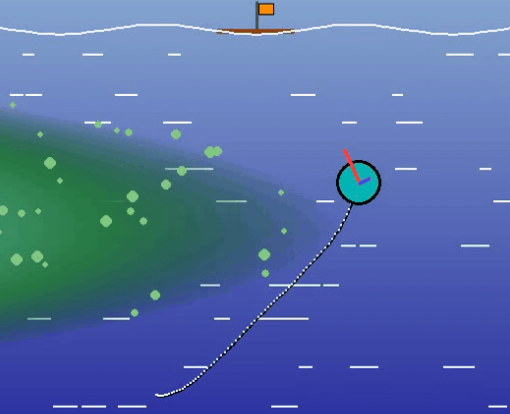}
  \end{minipage}\hfill
  \begin{minipage}{0.48\linewidth}
    \centering
    \includegraphics[width=\linewidth]{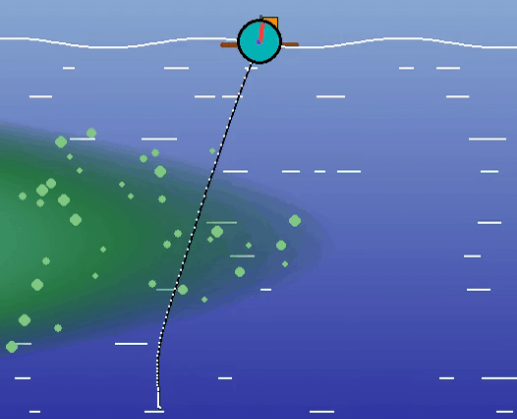}
  \end{minipage}
  \caption{
    Visualizations for the Contaminated-Zone AUV Navigation task.
    Left: environment layout.
    The contaminated region is highlighted in green, and the goal is
    located at the water surface at $(0, 4)$.
    Right: sample trajectory of the baseline policy.
    The baseline reaches the goal region but passes through the
    contaminated area, illustrating its suboptimal behavior.
  }
  \label{fig:underwater_drone_visuals}
\end{figure}

\paragraph{Environment Description} The Contaminated-Zone AUV
Navigation task shown in the left panel of
\Cref{fig:underwater_drone_visuals} is governed by the following
system of differential equations:
\begin{equation}
  \begin{aligned}
    \dot{x}   & = v_x \quad\dot{y} = v_y  \quad \dot{\vartheta} =
    \omega \quad \dot{\omega} = \tfrac{d_{\text{offset}} \cdot
    F_{\text{lat}}}{I} \\
    \dot{v}_x & = \tfrac{F_{\text{long}} \cos \vartheta -
    F_{\text{lat}} \sin \vartheta - C_d \|v\|v_x}{m}
    \\
    \dot{v}_y & = \tfrac{F_{\text{long}} \sin \vartheta +
    F_{\text{lat}} \cos \vartheta - C_d \|v\|v_y - mg}{m}
  \end{aligned}
\end{equation}
where the state vector $\state = (x, y, \vartheta, v_x, v_y, \omega)
\in \mathbb{R}^6$ represents the AUV's position coordinates $(x, y)$,
orientation angle $\vartheta$, linear velocities $(v_x, v_y)$, and
angular velocity $\omega$.
The control input $a = (F_{\text{long}}, F_{\text{lat}}) \in [-1, 1]
\times [-0.5, 0.5]$ is two-dimensional, representing the longitudinal
and lateral thrust forces.

The system parameters are: mass $m = 1$, moment of inertia $I = 0.1$,
gravitational acceleration $g = 0.5$, drag coefficient $C_d = 0.05$,
and lateral offset distance $d_{\text{offset}} = 0.2$ for torque generation.


\paragraph{Objective and Reward Function} The reward function
encourages the AUV to reach the target while penalizing excessive
velocities and contaminated area intrusion:
\begin{equation*}
  \reward(\state, \action) = -\frac{(y - 4)^2}{4} - \frac{x^2}{4} -
  \frac{v_x^2}{20} - \frac{v_y^2}{20} - \frac{\omega^2}{100} - 5
  \cdot \mathbb{I}\{(x, y) \in \mathcal{C}\}
\end{equation*}
Here the two interpretations coincide: the goal set is specified
explicitly as $\G = \{(x, y) \in \mathbb{R}^2 \mid \|x\| < 0.4 \text{
and } y \geq 4 \}$, and it also corresponds to the high-reward target
area at the water surface.
The contaminated region is a parabolic area $\mathcal{C} = \{(x, y)
  \in \mathbb{R}^2 \mid \frac{x}{0.81} + \frac{(y - 2)^2}{0.36} \leq
1\}$ that the agent must avoid, depicted in green in
\Cref{fig:underwater_drone_visuals}.
The episode length is 1500 steps.
Every step is integrated with a time step of 0.02 seconds.

\paragraph{Initial Conditions} The initial state distribution is
uniform over the domain:
\begin{equation}
  \begin{aligned}
    \transit_0(\bullet) & \sim \text{Uniform}\big([-2, 2] \times [0,
      4/3] \times [9\pi/20, 11\pi/20] \\
    & \quad \times [-0.2, 0.2] \times [-0.2, 0.2] \times [-0.2, 0.2]\big),
  \end{aligned}
\end{equation}
ensuring the AUV starts in the lower portion of the environment with
moderate initial velocities and near-vertical orientation.


\paragraph{Baseline Policy Design}
The baseline policy employs two PD controllers with coordinate
transformation to direct the AUV toward the goal.
While this approach successfully reaches the goal set $\G$, it does
not account for the contaminated region, which often results in
suboptimal trajectories through the penalty area.
A sample trajectory demonstrating this behavior is shown in the right
panel of \Cref{fig:underwater_drone_visuals}.
The complete implementation can be found in the repository: \gitrepo.
An animated visualization comparing trajectory behaviors of the
baseline policy and the final trained policy produced by the proposed
method is also available in the repository.

\subsection{Treasure-Collecting Robot}\label{sec:kin-robot}

\begin{figure}[!b]
  \centering
  \begin{minipage}{0.48\linewidth}
    \centering
    \includegraphics[width=\linewidth]{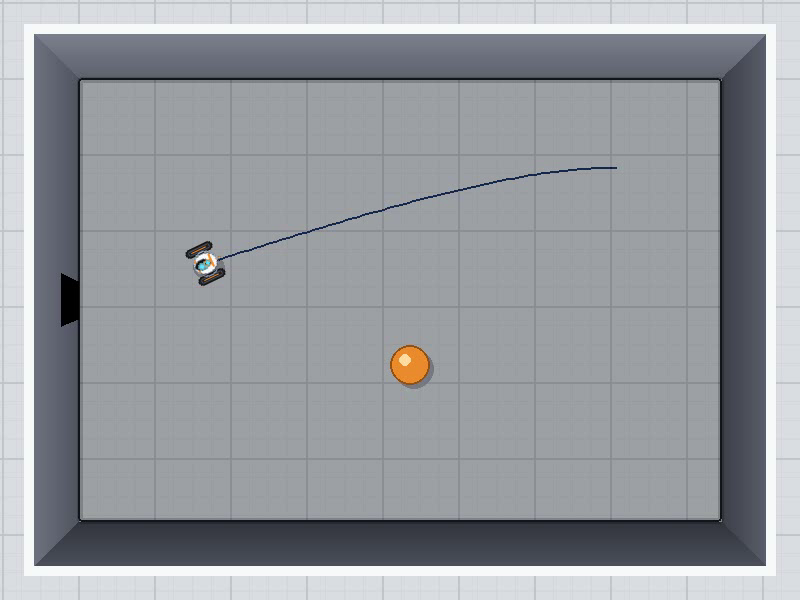}
  \end{minipage}\hfill
  \begin{minipage}{0.48\linewidth}
    \centering
    \includegraphics[width=\linewidth]{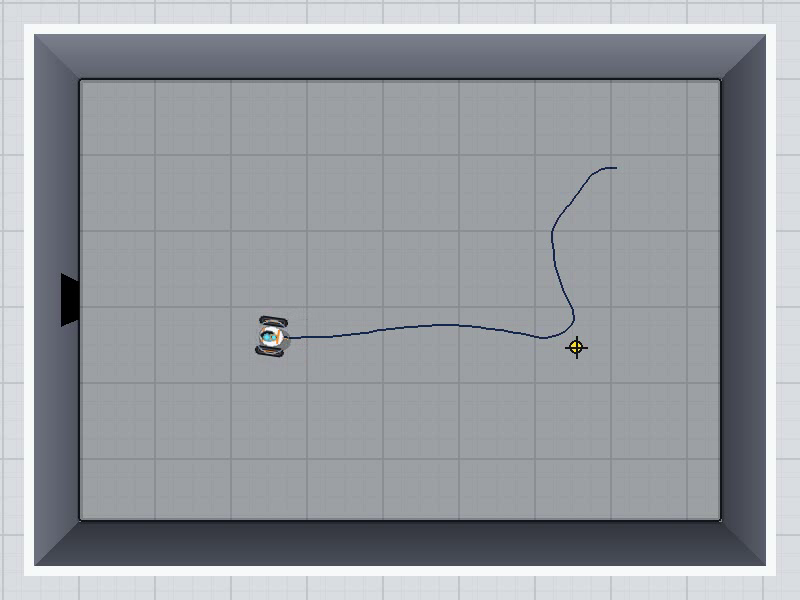}
  \end{minipage}
  \caption{
    Visualizations of the Treasure-Collecting Robot task.
    Left: representative behavior of the nominal policy.
    The black marker on the left boundary indicates the goal, and the
    orange disk denotes the collectible treasure.
    The baseline policy steers directly toward the goal and does not
    intercept the treasure.
    Right: desired behavior induced by the learned policy.
    The learned policy first deviates toward the collectible treasure
    and then returns toward the goal, instead of following a direct
    goal-seeking trajectory.
  }
  \label{fig:kin_robot_rollouts}
\end{figure}

\paragraph{Environment Description}
The robot moves with a constant linear speed.
Its dynamics are given by
\begin{equation*}
  \dot{x} = v \cos \vartheta, \quad \dot{y} = v \sin \vartheta, \quad
  \dot{\vartheta} = \omega,
\end{equation*}
where $(x, y) \in [0,1]^2$ are the robot's coordinates, $\vartheta
\in [-\pi, \pi)$ is the robot's heading, $v\equiv 0.15 \mathrm{m/s}$
denotes the robot's constant linear speed, and $\omega \in [-\pi,
\pi)$ is the angular velocity serving as the control input.

During the episode the robot must collect a treasure whose position
$(x^{\mathrm{tr}}, y^{\mathrm{tr}}) \in [0,1]^2$ is known to the robot.
The treasure is considered collected when the robot enters a
neighborhood of radius $0.05$ around the treasure, i.e.,
\[
  \left(x - x^{\mathrm{tr}}\right)^2 + \left(y -
  y^{\mathrm{tr}}\right)^2 < 0.05^2.
\]
The treasure moves stochastically in the workspace.
Let $(x_t^{\mathrm{tr}}, y_t^{\mathrm{tr}})$ and $(v_t^x, v_t^y)$
denote its position and velocity at time $t$, and let $\Delta t =
0.05\,\mathrm{s}$.
Its initial velocity $(v_0^x, v_0^y)$ has a random heading and fixed
speed $v_{\max} = 0.12$.
At each time step, the velocity is perturbed by random variables
$\epsilon_t^x, \epsilon_t^y \sim \mathcal N(0, 0.03^2)$, which are
independent across components and time, and the resulting velocity is
clipped to magnitude $v_{\max}$.
The velocity and position updates are given by
\begin{equation*}
  v^x_{t+1}\!=\! \frac{v^x_t+\epsilon^x_t}
  {\max\!\left(1,\!\frac{\sqrt{(v^x_t+\epsilon^x_t)^2 +
  (v^y_t+\epsilon^y_t)^2}}{v_{\max}}\right)}\;
  v^y_{t+1}\!=\!\frac{v^y_t+\epsilon^y_t}
  {\max\!\left(1,\!\frac{\sqrt{(v^x_t+\epsilon^x_t)^2 +
  (v^y_t+\epsilon^y_t)^2}}{v_{\max}}\right)}
\end{equation*}
\begin{equation*}
  x^{\mathrm{tr}}_{t+1} = x^{\mathrm{tr}}_t + \Delta t\, v^x_{t}\quad
  y^{\mathrm{tr}}_{t+1} = y^{\mathrm{tr}}_t + \Delta t\, v^y_{t}.
\end{equation*}
The update is followed by reflection at the workspace boundaries for
$x^{\mathrm{tr}}_{t+1}$ and $y^{\mathrm{tr}}_{t+1}$, implemented as
component-wise velocity sign flips.

After collecting the treasure, the robot must reach a goal region
that is specified explicitly and is also favored by the
distance-to-goal term in the reward.
The goal set is defined as
\[
  \G = \left\{(x, y) \in [0,1]^2 \;\middle|\; \left(x -
  x^{\mathrm{g}}\right)^2 + \left(y - y^{\mathrm{g}}\right)^2 < 0.05^2 \right\},
\]
where $(x^{\mathrm{g}}, y^{\mathrm{g}}) = (0, 0.5)$ is the goal position.

The state vector is 7-dimensional and is defined as
\[
  \state = (x, y, \cos \vartheta, \sin \vartheta, x^{\mathrm{tr}},
  y^{\mathrm{tr}}, I^{\mathrm{tr}}),
\]
where $I^{\mathrm{tr}} \in \{0,1\}$ is a binary indicator of treasure
availability:
\[
  I^{\mathrm{tr}} =
  \begin{cases}
    1, & \text{if the treasure is available for collection}, \\
    0, & \text{if the treasure has already been collected}.
  \end{cases}
\]
The dynamics are integrated using an explicit Euler scheme at 20 Hz.
Each episode lasts at most 1000 steps (50 s) and terminates early if
the goal set $\G$ is reached.
Representative rollouts for this task are shown in
\Cref{fig:kin_robot_rollouts}.

\paragraph{Reward Function and Objective}
The reward function is defined as
\begin{multline*}
  r(s,a) = -\left\lVert (x^{\mathrm{g}}, y^{\mathrm{g}}) - (x, y)
  \right\rVert_2 \\
  + 50 \cdot I^{\mathrm{tr}}
  \cdot \mathbb{I}\!\left\{
    \left\lVert (x, y) - (x^{\mathrm{tr}}, y^{\mathrm{tr}})
    \right\rVert_2 < 0.05
  \right\},
\end{multline*}
where $\lVert \bullet \rVert_2$ denotes the Euclidean norm.

Thus, the task rewards both collecting the treasure and reaching the
goal as quickly as possible.

\paragraph{Initial Conditions}

The initial state is sampled according to the following distribution:

\[
  \begin{aligned}
    \transit_0(\bullet) & \sim \text{Uniform}\big( [0.7, 0.9] \times
      [0.1, 0.9] \times \{-1\} \times \\ & \quad \{0\} \times [0.15,
    0.85] \times [0.05, 0.95] \times \{1\}\big)
  \end{aligned}
\]
\paragraph{Baseline Policy Design}
The baseline employs a simple geometric steering law that turns the
robot toward the provided target point (goal).
From the observation $[x, y, \cos\vartheta, \sin\vartheta, \ldots]$,
the controller reconstructs the current heading as $\vartheta =
\mathrm{atan2}(\sin\vartheta,\cos\vartheta)$ and computes the desired
heading to the goal as $\vartheta^\star =
\mathrm{atan2}(y^{\mathrm{g}}-y,\, x^{\mathrm{g}}-x)$.
The wrapped angle error is $e = (\vartheta^\star-\vartheta+\pi)\bmod
2\pi - \pi$, and the action is the clipped proportional control
\[
  \omega = \mathrm{clip}(k_{\text{turn}}\, e,\; -\omega_{\max},\;
  \omega_{\max}).
\]
The parameters are set to $k_{\text{turn}}=1$ and $\omega_{\max}=\pi$.
When the robot is (numerically) at the goal, $\vartheta^\star$ is set
to $\vartheta$ to avoid an ill-defined direction.
This baseline reliably reaches the goal but optimizes only the
goal-reaching objective; it is not designed to capture high-reward
treasure, which results in suboptimal behavior with respect to the
full reward function.
This behavior is illustrated in the left panel of \Cref{fig:kin_robot_rollouts}.

\subsection{Experimental results}\label{sec:exp_results}
\paragraph{Experimental Setup} The proposed method is evaluated
with two backbones, TD3 and SAC, along with the corresponding vanilla
and residual variants for each backbone.
All algorithms are trained for 3M environment steps across ten
independent random seeds for each algorithm-environment pair to
ensure statistical reliability.
The TD3 and SAC implementations are both sourced from the CleanRL
library~\cite{huang2022cleanrl} and use the same environment
interfaces and evaluation protocol.
No backbone-specific hyperparameter tuning is performed: for both TD3
and SAC, the default CleanRL hyperparameters are used and kept fixed
within each backbone family.
For the proposed method, only the method-specific parameters are selected.
$\nu=0.01$ is set to a small value, so that the critic gate does not
become overly conservative.
The relaxation-schedule parameters $\relprob_0$ and $\lambda_0$ are
chosen so that $\relprob_0 \sum_{t=0}^{T-1}\lambda_0^t / T \approx 0.2$.
This yields $\relprob_0=0.8$ and $\lambda_0=0.995$ for the
Contaminated-Zone AUV Navigation environment, and $\relprob_0=0.9$
and $\lambda_0=0.96$ for the Treasure-Collecting Robot environment.
The transition time $T_{\text{tran}}$ is set to $2.7$M; the baseline
policy is never invoked after this point, and the learning policy
operates independently from $2.7$M to $3.0$M timesteps.

\subsubsection{Evaluation}

Several types of evaluation are introduced for better understanding
the performance of the proposed method across both TD3 and SAC backbones:
\begin{itemize}
  \item \textit{Learning curves (episode return).} Episode return is
    defined as the cumulative sum of rewards within an episode.
    Learning curves report episode return versus the total number of
    environment timesteps and are presented in
    \Cref{fig:episode_return,fig:sac_episode_return}.
\end{itemize}

\begin{itemize}
  \item \textit{Goal reaching during training.} To demonstrate the
    effectiveness of the proposed method in terms of goal-reaching
    performance \emph{throughout training}, a rolling goal-reaching
    rate is computed for each random seed.
    Specifically, for every completed episode, the binary
    goal-reaching indicator is averaged over the most recent 75
    episodes in the same seed, using all available earlier episodes
    at the beginning of training.
    The resulting per-seed rolling curves are aligned on a regular
    timestep grid, and the median across seeds together with the
    interquartile range is then reported at each grid point.
  \item \textit{Final-stage metrics.} Since episode return often
    provides limited interpretability in reinforcement learning,
    additional task-oriented metrics are reported for both
    environments in~\Cref{tab:metrics} and~\Cref{tab:metrics_kin_robot}.
    These metrics are computed over the final stage of training,
    i.e., for timesteps from $2.7$M to $3.0$M (after the transition
    time $T_{\mathrm{tran}} = 2.7$M).
    The following metrics are used:
    \begin{itemize}
      \item \textit{Goal-reaching rate (both environments):}
        percentage of episodes in which the goal set is reached.
        Higher values are better (best is $100\%$).

      \item \textit{Treasure-collection rate (Treasure-Collecting
        Robot):} percentage of episodes in which the treasure is collected.
        Since successful completion requires both collecting the
        treasure and reaching the goal, \Cref{tab:metrics_kin_robot}
        reports both rates.
        A combined success score is also reported, defined as the
        average of the goal-reaching and treasure-collection rates.
        Higher values are better (best is $100\%$).

      \item \textit{Avoidance score (Contaminated-Zone AUV
        Navigation):} maximum penetration depth into $\mathcal{C}$
        during an episode,
        \[
          \max_{t \in \text{episode}} d\left((x_t, y_t),
          \bar{\mathcal{C}}\right),
        \]
        where $d(\cdot,\cdot)$ denotes the Euclidean distance,
        $\bar{\mathcal{C}}$ is the closed complement of
        $\mathcal{C}$, and $(x_t,y_t)$ is the agent position at time $t$.
        Lower values are better (best is $0$).
    \end{itemize}

\end{itemize}
\paragraph{Diagnostic plots}
To illustrate how control is gradually transferred from the baseline
policy to the learning policy, two diagnostics are reported.
The first is the fraction of actions selected by the learning policy
per episode.
The second is the evolution of the schedule parameters $\relprob$ and $\lambda$.
Figure~\ref{fig:fraction_baseline} shows the learning-policy call
fraction during training, averaged over ten independent random seeds,
while Figure~\ref{fig:schedule_parameters} shows the corresponding
transition schedule.
\begin{figure}[h]
  \centering
  \includegraphics[width=0.9\linewidth]{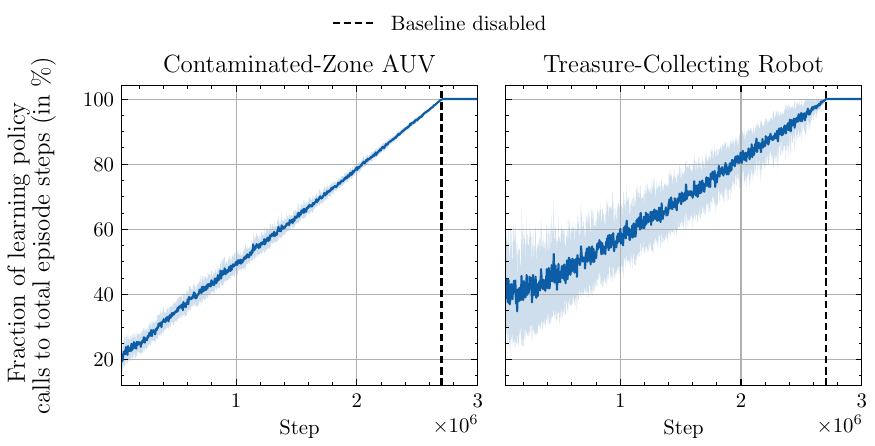}
  \caption{
    Evolution of the fraction of learning policy calls per episode
    during training.
    Results are averaged over ten independent random seeds, with
    shaded regions indicating standard deviation.
    At the beginning of training, the baseline policy produces most
    of the actions, so the learning-policy call fraction is low.
    As training progresses, agency is gradually transferred to the
    learning policy, with the learning-policy call fraction
    increasing approximately linearly---neither too abruptly nor too slowly.
    The vertical dashed line marks the Baseline disabled point, after
    which this fraction reaches $100\%$ and the baseline policy is no
    longer called (\Cref{alg:enhance}).
  }
  \label{fig:fraction_baseline}
\end{figure}

\begin{figure}[h]
  \centering
  \includegraphics[width=0.9\linewidth]{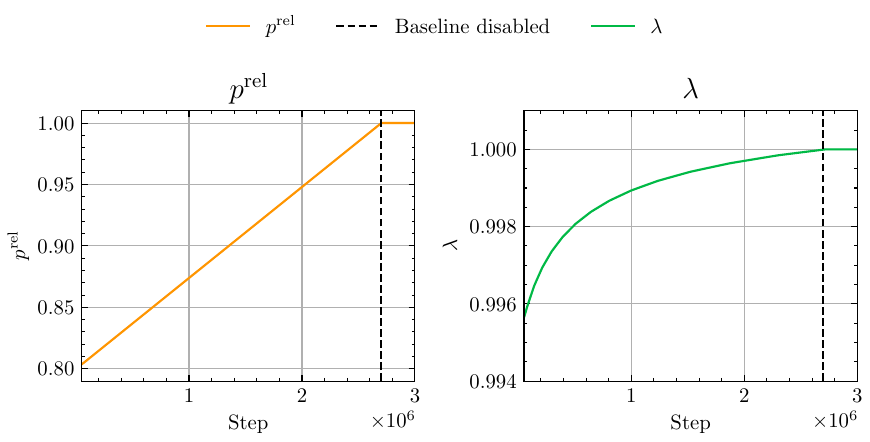}
  \caption{
    Evolution of the schedule parameters $\relprob$ and $\lambda$
    during training in the Contaminated-Zone AUV Navigation environment.
    The curves show the logged schedule values; since the schedule is
    deterministic, the same values are obtained for all ten
    independent random seeds.
    Both parameters increase monotonically toward their terminal
    value of one; the vertical dashed line marks the Baseline disabled point.
  }
  \label{fig:schedule_parameters}
\end{figure}

\subsubsection{Results}
\begin{figure*}[!t]
  \centering

  \begin{minipage}{0.48\textwidth}
    \centering
    \includegraphics[width=\linewidth]{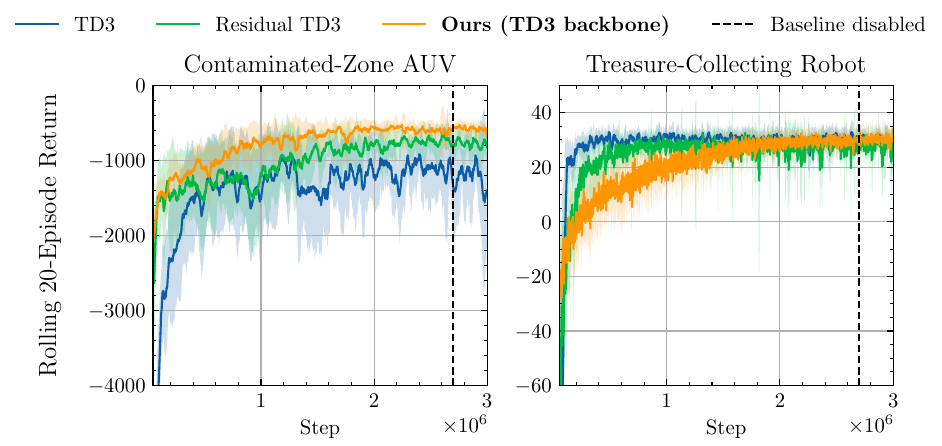}
    \vspace{-6pt}
    \caption{
      Episode return comparison for TD3-based methods.
      Results are averaged over ten independent random seeds for each
      algorithm-environment pair, with standard deviation bands.
      The vertical dashed line marks the Baseline disabled point.
    }
    \label{fig:episode_return}
  \end{minipage}\hfill
  \begin{minipage}{0.48\textwidth}
    \centering
    \includegraphics[width=\linewidth]{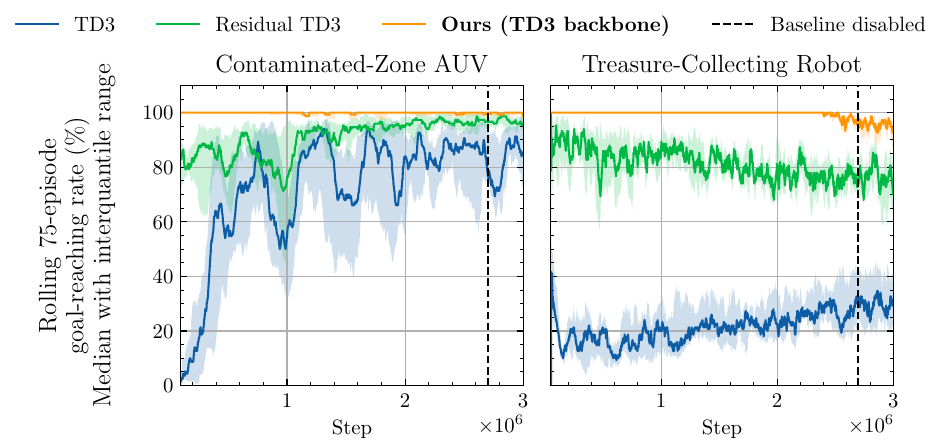}
    \caption{
      Goal-reaching rate comparison for TD3-based methods.
      The solid line denotes the median across ten independent random
      seeds, while the shaded region denotes the interquartile range.
      The vertical dashed line marks the Baseline disabled point.
    }
    \label{fig:goal_reaching_rates}
  \end{minipage}
\end{figure*}

\begin{figure*}[!t]
  \centering

  \begin{minipage}{0.48\textwidth}
    \centering
    \includegraphics[width=\linewidth]{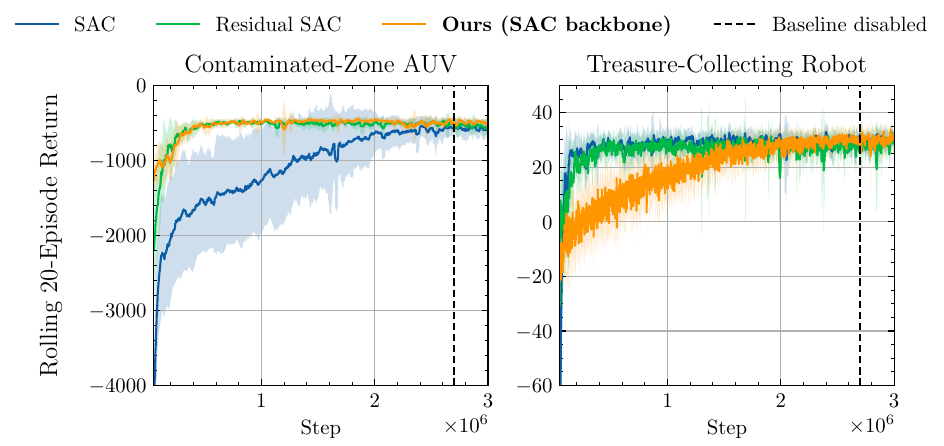}
    \vspace{-6pt}
    \caption{
      Episode return comparison for SAC-based methods.
      Results are averaged over ten independent random seeds for each
      algorithm-environment pair, with standard deviation bands.
      The vertical dashed line marks the Baseline disabled point.
    }
    \label{fig:sac_episode_return}
  \end{minipage}\hfill
  \begin{minipage}{0.48\textwidth}
    \centering
    \includegraphics[width=\linewidth]{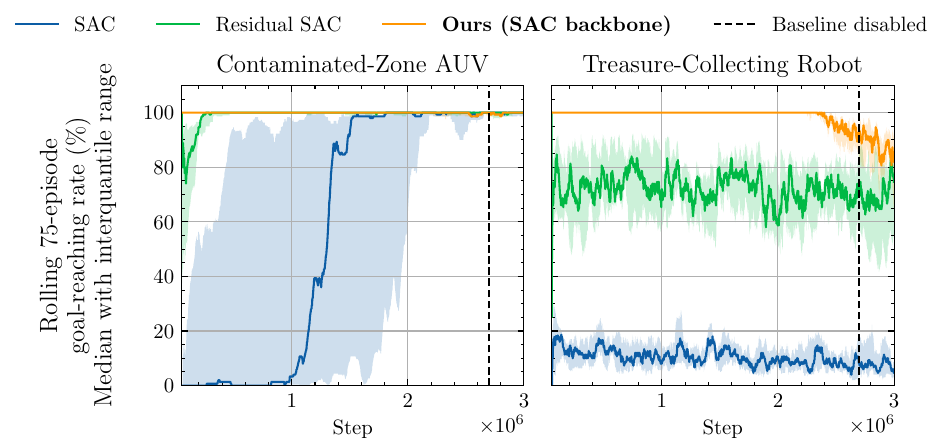}
    \caption{
      Goal-reaching rate comparison for SAC-based methods.
      The solid line denotes the median across ten independent random
      seeds, while the shaded region denotes the interquartile range.
      The vertical dashed line marks the Baseline disabled point.
    }
    \label{fig:sac_goal_reaching_rates}
  \end{minipage}
\end{figure*}

\begin{table*}[!t]
  \centering

  \begin{minipage}{0.48\textwidth}
    \centering \scriptsize \setlength{\tabcolsep}{2pt}
    \caption{
      \textbf{Final-stage metrics in the Contaminated-Zone AUV
      Navigation environment.} TD3- and SAC-based methods are
      reported in the same table.
      Results correspond to episodes occurring between 2.7M and 3.0M
      environment steps and are aggregated over ten independent random seeds.
      \textit{Goal reached (\%)} denotes the percentage of episodes
      in which the agent successfully reaches the goal set $\G$
      within the episode horizon.
      \textit{Avoidance score} is defined as the maximum penetration
      depth into contaminated region $\mathcal{C}$ during an episode,
      $\max_{t \in \text{episode}} d((x_t,y_t),\bar{\mathcal{C}})$,
      where $d(\cdot,\cdot)$ is the Euclidean distance and
      $\bar{\mathcal{C}}$ is the closed complement of $\mathcal{C}$.
      Lower avoidance score indicates safer behavior (less intrusion
      into $\mathcal{C}$), while higher goal-reaching rate indicates
      better task success.
      All values are shown as \textit{mean $\pm$ standard deviation}.
    }\label{tab:metrics}
    \begin{tabular}{lcc}
      \toprule Method & Goal reached (\%) & Avoidance score
      \\ \midrule \textbf{Ours (TD3 backbone)} & 99.15 $\pm$ 0.58 &
      0.0100 $\pm$ 0.0031 \\ Residual TD3 & 95.60 $\pm$ 3.75 & 0.0207
      $\pm$ 0.0132 \\ TD3 & 76.35 $\pm$ 25.47 & 0.0150 $\pm$ 0.0104
      \\ \midrule \textbf{Ours (SAC backbone)} & 99.45 $\pm$ 0.50 &
      0.0052 $\pm$ 0.0031 \\ Residual SAC & 99.10 $\pm$ 1.07 & 0.0115
      $\pm$ 0.0076 \\ SAC & 90.25 $\pm$ 29.43 & 0.0065 $\pm$ 0.0048
      \\ \midrule Baseline policy & 100 $\pm$ 0 & 0.54 $\pm$ 0.18 \\ \bottomrule
    \end{tabular}
  \end{minipage}\hfill
  \begin{minipage}{0.48\textwidth}
    \centering \scriptsize \setlength{\tabcolsep}{1pt}
    \caption{
      \textbf{Final-stage metrics in the Treasure-Collecting Robot
      environment.} TD3- and SAC-based methods are reported in the same table.
      Results correspond to episodes occurring between 2.7M and 3.0M
      environment steps and are aggregated over ten independent random seeds.
      \textit{Goal reached (\%)} denotes the percentage of episodes
      in which the robot reaches the goal set $\G$ within the episode horizon.
      \textit{Treasure collected (\%)} denotes the percentage of
      episodes in which the robot successfully collects the required
      treasure at least once during the episode.
      Since the task requires achieving both sub-goals, the
      \textit{Combined} metric is defined as
      $\tfrac{1}{2}(\textit{Goal reached} + \textit{Treasure collected})$.
      Higher values indicate better performance.
      All values are shown as \textit{mean $\pm$ standard deviation}.
    }\label{tab:metrics_kin_robot}
    \begin{tabular}{lccc}
      \toprule Method & Goal reached (\%) &
      \makecell{Treasure\\collected (\%)} & Combined \\ \midrule
      \textbf{Ours (TD3 backbone)} & 94.11 $\pm$ 3.85 & 99.78 $\pm$
      0.10 & 96.95 $\pm$ 1.92 \\ Residual TD3 & 76.15 $\pm$ 6.01 &
      99.27 $\pm$ 0.76 & 87.71 $\pm$ 2.92 \\ TD3 & 31.69 $\pm$ 13.27
      & 99.89 $\pm$ 0.14 & 65.79 $\pm$ 6.67 \\ \midrule \textbf{Ours
      (SAC backbone)} & 85.05 $\pm$ 11.01 & 99.56 $\pm$ 0.35 & 92.31
      $\pm$ 5.53 \\ Residual SAC & 66.62 $\pm$ 18.70 & 98.86 $\pm$
      1.09 & 82.74 $\pm$ 9.48 \\ SAC & 11.04 $\pm$ 10.36 & 99.78
      $\pm$ 0.30 & 55.41 $\pm$ 5.22 \\ \midrule Baseline policy & 100
      $\pm$ 0 & 17.30 $\pm$ 37.84 & 58.65 $\pm$ 18.92 \\ \bottomrule
    \end{tabular}

  \end{minipage}

\end{table*}

The results support the following observations.
\paragraph{Goal-reaching} \Cref{alg:enhance} achieves the
highest goal-reaching rates across all baselines in both environments.
\begin{itemize}
  \item During training, the proposed method outperforms all other
    algorithms evaluated alongside it in terms of goal-reaching rate (see
    \Cref{fig:goal_reaching_rates,fig:sac_goal_reaching_rates}).
    In particular, the goal-reaching rate is already high at the
    beginning of training.
    This is due to the design of the proposed method, which
    effectively embeds the baseline policy into RL training process.
    A formal interpretation of this behavior is provided
    in~\Cref{sec:theoretical_analysis}.
    Equivalently, the influence of the learning policy is small at
    the beginning of training and then increases at a moderate pace
    (approximately linearly on average; see the right panel of
    \Cref{fig:underwater_drone_visuals}).
  \item After the transition ($t \geq T_{\mathrm{train}} =
    2.7\text{M}$), when the baseline influence becomes exactly zero,
    the learning policy---implemented as a standalone neural network
    --- operates independently and still attains the highest
    goal-reaching rates.
    This holds in comparison with the corresponding residual and
    vanilla backbone baselines, as reported in \Cref{tab:metrics} and
    \Cref{tab:metrics_kin_robot}.
\end{itemize}
The slight decrease in goal-reaching rate observed near the end of
training is consistent with the intended transition mechanism.
During the early and intermediate stages, the baseline policy still
has substantial influence and can often recover trajectories in which
the learning policy has moved the system away from an easy path to the goal.
Near the transition time, this corrective influence becomes
negligible: the learning policy is selected almost always, and within
a finite episode there may be too little baseline intervention left
to compensate for occasional mistakes.
This finite-horizon effect is precisely why the early-stage
goal-reaching theorem in \Cref{sec:theoretical_analysis} is used as
an interpretation rather than as a direct performance guarantee: in
experiments, goal reaching is always judged over a finite episode
length, so any empirical claim that a policy reaches the goal
necessarily depends on whether the observation horizon is long enough
for the relevant recovery mechanism to act.
The same point is reflected in the trajectory-distance transfer
analysis in \Cref{thm:transfer}: after moving from the
baseline-supported regime to the baseline-free regime, the
finite-horizon goal-reaching probability may degrade by a term
controlled by the distance between the corresponding trajectory distributions.

Nevertheless, training on trajectories that predominantly reach the
goal has a lasting effect.
After the baseline is removed, the standalone learning policy
achieves higher final goal-reaching metrics than the other evaluated algorithms.
This is the empirical advantage targeted by the proposed method: the
baseline shapes the training distribution toward successful
trajectories, and this influence persists even in the final
baseline-free regime.

\paragraph{Final-stage metrics}
Beyond goal-reaching, the proposed approach performs well on
task-specific metrics (see~\Cref{tab:metrics} and~\Cref{tab:metrics_kin_robot}).
In particular, the avoidance score (Contaminated-Zone AUV Navigation)
and the treasure collection rate (Treasure-Collecting Robot) are both strong.
Combined with the best goal-reaching performance, this indicates that
the proposed method finds better overall behaviors in both cases: it
reliably reaches the goal set while also optimizing the task
objective (avoiding the contaminated region and collecting the
treasure, respectively), outperforming the corresponding vanilla and
residual backbone baselines.

\paragraph{Cumulative reward and training} The learning curves
in \Cref{fig:episode_return,fig:sac_episode_return} show that the
proposed approach achieves either superior returns or, in the worst
case, comparable returns.

\paragraph{Conclusion}
Overall, the proposed approach effectively \emph{improves} the RL
training procedure by producing a standalone neural policy that
outperforms both from-scratch methods and domain-specific residual
methods that also embed a baseline policy into the RL training process.
Moreover, the empirical results show that, when the baseline policy
is capable of solving the task, the proposed method can reach the
goal from the very beginning of training, even when the underlying
learning policy is still completely undertrained.
This yields stronger goal-reaching capability than the other
evaluated approaches.
The vanilla TD3 and SAC baselines are trained without access to the
baseline policy, providing reference points that make the cost of
learning from scratch explicit under the same training budget.
Their underperformance relative to the residual variants and to the
variants equipped with the proposed approach is therefore expected,
as the latter incorporate prior knowledge from the baseline policy
into the learning process.
As noted in the introduction, tuning RL methods can be difficult:
achieving strong performance often requires careful reward shaping
and extensive hyperparameter tuning.
The experiments suggest that when a working baseline policy is
already available, it can be leveraged to substantially simplify this
process and yield a policy that not only solves the task but also
does so more efficiently than a policy trained from scratch.

\section{Ablation and sensitivity studies}\label{sec:ablation_studies}

\subsection[Effect of nu: critic-value tracking]{Effect of $\nu$:
critic-value tracking}

The intuition behind tracking the best critic value $Q_t^{\dagger}$
is discussed in \Cref{par:intuition}; this section empirically
evaluates its practical contribution.
To this end, the tracking mechanism is ablated by disabling the
deterministic acceptance condition based on the improvement margin.
Concretely, the default setting $\nu=0.01$ is replaced by
$\nu=\infty$, which makes the condition
\[
  Q^{w_t}(\State_t,\Action^{\learningpolicy}_t) \ge Q_t^{\dagger} + \nu
\]
never satisfied.
As a result, the algorithm never performs ``certain'' acceptance
based on outrunning the episode-local benchmark and instead always
operates in the probabilistic acceptance regime.

The learning curves for this ablation are shown in
\Cref{fig:ablation_studies_kin_robot}.
Across all considered environments, enabling tracking yields
consistently better or, at worst, comparable performance throughout training.
In contrast, removing tracking leads to a systematic degradation of
returns and typically increases variability, indicating that the
benchmark-based test provides a stabilizing effect in practice.
While the method remains operational without tracking (i.e., it does
not collapse), its sample-efficiency and final performance are
noticeably reduced, which supports the conclusion that critic-value
tracking is an essential component of the proposed approach.

In addition to the return curves, the analysis also examines how
often the algorithm falls back to the baseline policy.
This effect is particularly pronounced in the Treasure-Collecting
Robot environment: when tracking is disabled, the number of baseline
policy calls increases substantially (see
\Cref{fig:ablation_baseline_policy_calls_kin_robot}).
This behavior is consistent with the intended role of tracking: by
recognizing confident improvements via the benchmark, the method more
readily commits control to the learning policy; without this
mechanism, the acceptance decisions become more conservative on
average, triggering baseline interventions more frequently and
slowing down effective transfer of agency.

\begin{figure}[h]
  \centering
  \includegraphics[width=\linewidth]{
    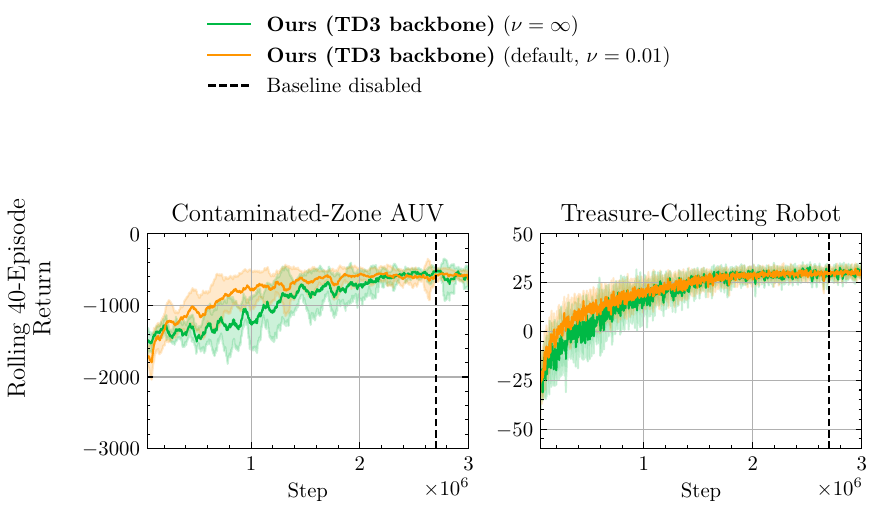
  }
  \caption{
    Comparison of training performance with critic-value tracking
    enabled ($\nu=0.01$) versus disabled (ablation with $\nu=\infty$).
    The proposed method is evaluated on top of a TD3 backbone.
    The vertical dashed line marks the Baseline disabled point in
    \Cref{alg:enhance}.
    Throughout training, tracking is consistently beneficial,
    achieving superior or comparable returns relative to the ablated variant.
  }
  \label{fig:ablation_studies_kin_robot}
\end{figure}

\begin{figure}[h]
  \centering
  \includegraphics[width=0.6\linewidth]{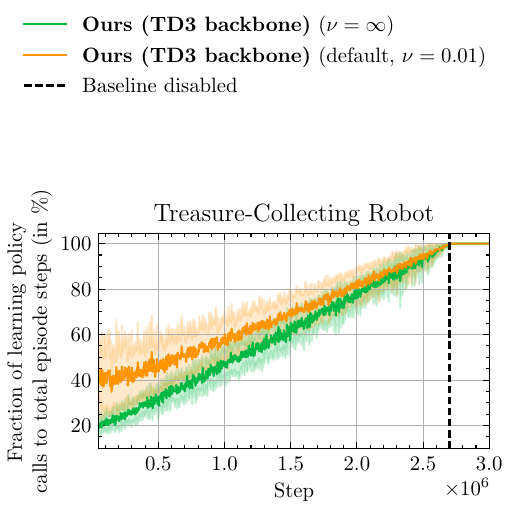}
  \caption{
    Fraction of learning policy calls during training with
    critic-value tracking enabled ($\nu=0.01$) versus disabled
    (ablation with $\nu=\infty$).
    The vertical dashed line marks the Baseline disabled point in
    \Cref{alg:enhance}.
    Critic-value tracking substantially increases the learning-policy
    call fraction, especially in the Treasure-Collecting Robot
    environment, indicating a more confident and efficient transfer of control.
  }
  \label{fig:ablation_baseline_policy_calls_kin_robot}
\end{figure}

\subsection[Effect of Ttran: baseline-removal time]{Effect of
$T_{\mathrm{tran}}$: baseline-removal time}

The next experiment evaluates how the time at which the baseline
policy is fully removed affects performance.
The baseline-removal time $T_{\mathrm{tran}}$ exposes a trade-off in
agency transfer.
Early removal can make the transition too abrupt: the learned policy
may not yet have accumulated sufficient closed-loop experience to
sustain the task without baseline support, leading to degraded
goal-reaching quality, as shown in \Cref{fig:calf_td3_drone_anneal_is_in_hole}.
Later baseline removal provides a longer protected learning phase,
improves the resulting outcome, and makes the agency transfer
smoother in the final results.

\begin{figure}[h]
  \centering
  \includegraphics[width=0.8\linewidth]{
    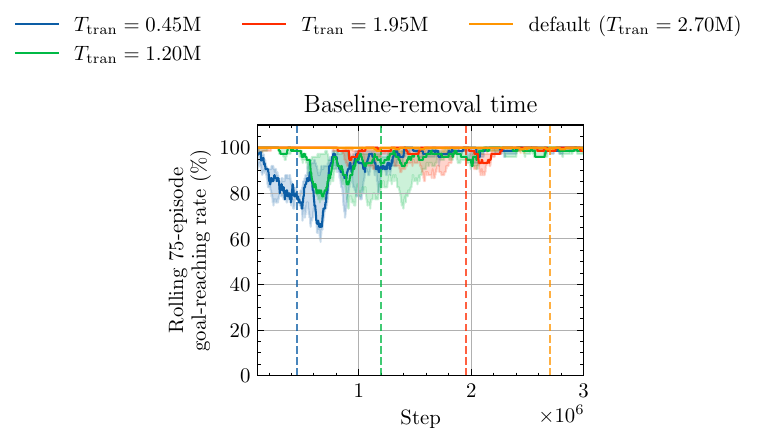
  }
  \caption{
    Rolling goal-reaching rate for the baseline-removal-time ablation
    of the proposed approach on top of a TD3 backbone in the
    Contaminated-Zone AUV Navigation environment.
    The plotted metric is the goal-reaching rate, reported as a
    rolling 75-episode median across five independent random seeds,
    with interquartile bands.
    Dashed vertical lines mark the corresponding Baseline disabled
    times $T_{\mathrm{tran}}$.
    All runs use the default configuration of the proposed approach
    on top of a TD3 backbone, with only $T_{\mathrm{tran}}$ varied.
  }
  \label{fig:calf_td3_drone_anneal_is_in_hole}
\end{figure}

\subsection[Effect of p0 and lambda0: relaxation schedule]{Effect of
$p_0$ and $\lambda_0$: relaxation schedule}

Finally, the sensitivity of the method to the within-episode
relaxation schedule, controlled by the initial probability $p_0$ and
the initial decay factor $\lambda_0$, is evaluated.
Larger values make the algorithm rely more strongly on the learning
policy from the beginning of training, whereas smaller values keep
the system closer to the baseline policy for longer.

If the initial trust in the learning policy is too high, the schedule
continues to amplify this trust during training, leaving less room
for baseline support in the early learning stages.
This can lead to worse final performance.
This effect is demonstrated by comparing the proposed method on top
of a TD3 backbone under the aggressive relaxation schedule
$p_0=1.0,\lambda_0=0.9995$ against the default schedule
($p_0=0.8,\lambda_0=0.995$) in the Contaminated-Zone AUV Navigation environment.
The return curves in
\Cref{fig:calf_td3_drone_schedule_ablation_returns} show that
starting with excessively large $p_0$ and $\lambda_0$ worsens performance.

A similar degradation can be observed in the Treasure-Collecting
Robot environment in terms of the goal-reaching rate
(\Cref{fig:calf_sac_robot_hyperparameter_sensitivity}).
There, the default schedule ($p_0=0.9,\lambda_0=0.96$) is compared
against the alternative schedule ($p_0=0.8,\lambda_0=0.995$), which
assigns substantially larger cumulative trust to the learning policy.
Indeed, the infinite-horizon schedule mass is
$p_0/(1-\lambda_0)=22.5$ for the default setting, but
$p_0/(1-\lambda_0)=160$ for the alternative setting, i.e., more than
seven times larger.

As a practical rule of thumb, following
\Cref{sec:transition-scheduling-strategy}, $p_0$ is initialized in
$[0.8, 1.0]$, and $\lambda_0$ is chosen so that
$\frac{p_0}{T}\sum_{t=0}^{T-1}\lambda_0^t \approx 0.2$.

\begin{figure}[h]
  \centering
  \includegraphics[width=.68\linewidth]{
    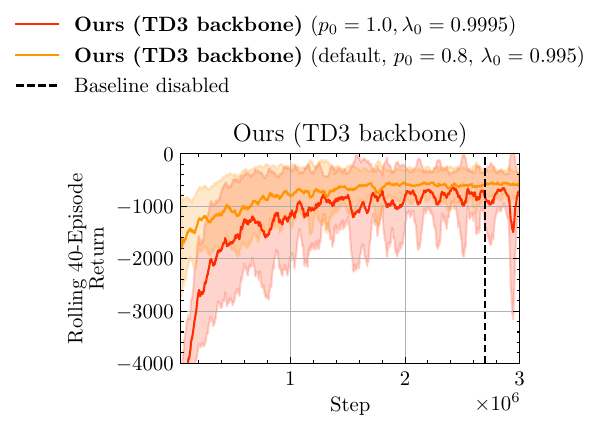
  }
  \caption{
    Return sensitivity of the proposed method on top of a TD3
    backbone in the Contaminated-Zone AUV Navigation environment
    under a relaxation schedule initialized too close to one,
    $p_0=1.0,\lambda_0=0.9995$, versus the default schedule.
    For the default TD3-backbone curve, $T_{\mathrm{tran}}=2.7$M,
    $p_0=0.8$, and $\lambda_0=0.995$.
    Curves show rolling 40-episode returns averaged over five
    independent random seeds, with standard deviation bands.
    The vertical dashed line marks the default Baseline disabled point.
  }
  \label{fig:calf_td3_drone_schedule_ablation_returns}
\end{figure}

\begin{figure}[h]
  \centering
  \includegraphics[width=.65\linewidth]{
    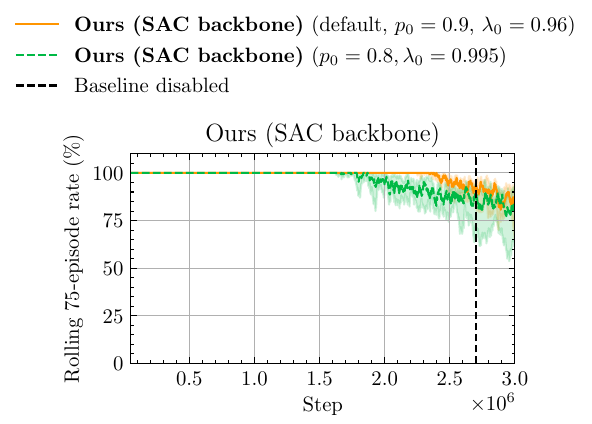
  }
  \caption{
    Rolling goal-reaching sensitivity of the proposed approach on top
    of a SAC backbone in the Treasure-Collecting Robot environment
    under the default setting versus the more conservative
    $p_0=0.8,\lambda_0=0.995$ setting.
    Curves show rolling 75-episode medians across ten independent
    random seeds, with interquartile bands.
    The vertical dashed line marks the default Baseline disabled point.
  }
  \label{fig:calf_sac_robot_hyperparameter_sensitivity}
\end{figure}

\appendix
\section{Proofs}
\label{app:proofs}

\subsection{Proof of Theorem~\ref{thm:goal-reaching}}
\label{app:proof_basic_goal_reaching}
\begin{proof}
  In the present proof, an episode is considered for which $\lambda <
  1$ at its beginning.
  Since the analysis is restricted to an intra-episode setting, the
  beginning time $\tau = 0$ may be set without loss of generality.
  By construction of \Cref{alg:enhance}, the parameter $\lambda$
  remains constant throughout the episode.

  The critical idea is to prove that the number of times the learning
  policy is triggered is bounded above \ie \( \sum_{t = 0}^{\infty}
    \mathbb{I}\{\Action_t = \Action^{\learningpolicy}_t\} < \infty.
  \) Moreover, $\Action_t = \Action^{\learningpolicy}_t$ if
  $Q^{w_t}(\State_t, \Action^{\learningpolicy}_t) \geq Q_t^{\dagger}
  + \nu$ or $U_t \leq \relprob \lambda^{t}$.
  Thus, $\sum_{t = 0}^{\infty} \mathbb{I}\{\Action_t =
  \Action^{\learningpolicy}_t\}$ is bounded above by $$ \sum_{t =
  0}^{\infty} \mathbb{I}\{Q^{w_t}(\State_t,
  \Action^{\learningpolicy}_t) \geq Q_t^{\dagger} + \nu\} + \sum_{t =
  0}^{\infty} \mathbb{I}\{U_t \leq \relprob \lambda^{t}\} $$

  Although $Q_0^{\dagger} = -\infty$ by initialization, the
  subsequent value $Q_1^{\dagger}$ is finite since the condition in
  line~\ref{alg:learning-3} is satisfied at $t = 0$.
  Since the sequence $\{Q^{\dagger}_t\}_{t\geq1}$ is non-decreasing
  and can only increase by increments of at least $\nu$, it follows
  that: \( \sum_{t = 0}^{\infty} \mathbb{I}\{Q^{w_t}(\State_t,
    \Action^{\learningpolicy}_t) \geq Q_t^{\dagger} + \nu\} \leq
    \frac{\bar{Q} - Q_1^{\dagger}}{\nu} + 1.
  \)

  The sum $\sum_{t = 0}^{\infty} \mathbb{I}\{U_t \leq \relprob
  \lambda^{t}\}$ is bounded by the Borel-Cantelli lemma since
  $\sum_{t = 0}^{\infty}\PP{U_t \leq \relprob \lambda^{t}} =
  \frac{\relprob}{1 - \lambda} < \infty$.
  Therefore, the total number of times the learning policy is
  triggered is bounded above.
  Consequently, there exists a time $t_0$ such that for all $t \geq
  t_0$, only the baseline policy is executed.
  The $\varepsilon$-improbable goal-reaching property of the
  resulting algorithm follows from the $\varepsilon$-improbable
  goal-reaching property of the baseline policy.
\end{proof}

\subsection{Proof of
Proposition~\ref{prop:functions_with_bounded_superlevel_sets}}
\label{app:proof_bounded_superlevel_sets}
\begin{proof}



  Let \(f^{\inf}:=\inf_{x\in\R^n} f(x)\in[-\infty,+\infty)\).
  The proof shows \(\textit{\ref{item:bounded_superlevel_sets}}
    \Leftrightarrow \textit{\ref{item:liminf}} \Leftrightarrow
  \textit{\ref{item:radially_unbounded}}\) by implications.

  \paragraph{\(\text{\ref{item:bounded_superlevel_sets}} \Rightarrow
  \text{\ref{item:liminf}}\)} Assume \(f(x)\) has bounded superlevel sets.

  \smallskip \emph{Case 1: \(f^{\inf}=-\infty\).}
  Fix any \(M>0\).
  Since \(f^{\inf}=-\infty\), there exists \(a\in f(\R^n)\) with \(a\le -M\).
  Boundedness of the \(a\)-superlevel set gives \(R>0\) such that
  \[
    \|x\|>R \Rightarrow f(x)<a\le -M.
  \]
  Hence \(\lim_{\|x\|\to\infty} f(x)=-\infty=f^{\inf}\).
  The inequality \(f(x)>f^{\inf}\) is automatic when \(f^{\inf}=-\infty\).

  \smallskip \emph{Case 2: \(f^{\inf}>-\infty\).}
  If \(f^{\inf}\in f(\R^{n})\), the \(f^{\inf}\)-superlevel set would
  equal \(\R^{n}\), contradicting boundedness.
  Thus \(f(x)>f^{\inf}\) for all \(x\),

  Let $a\in f(\mathbb{R}^n)$ be arbitrary with $f^{\inf}<a$.
  Because $f$ has bounded super-level sets, there is an $R>0$ such
  that $ f(x)\ge a \;\Longrightarrow\; \lVert x\rVert\le R, $ or
  equivalently, $ \lVert x\rVert>R \;\Longrightarrow\; f(x)<a.
  $ Because $f^{\inf}=\inf f(\mathbb{R}^n)$ and $f(x)>f^{\inf}$ for
  every $x\in\mathbb{R}^n$, one can pick $a$ arbitrarily close to $f^{\inf}$.

  \paragraph{\(\text{\ref{item:liminf}} \Rightarrow
  \text{\ref{item:bounded_superlevel_sets}}\)} Assume
  \(\lim_{\|x\|\to\infty} f(x)=f^{\inf}:=\inf f(\R^n)\) and
  \(f(x)>f^{\inf}\) for all \(x\).

  Let \(a\in f(\R^n)\).
  Since \(f(\R^n)\subset(f^{\inf},+\infty)\), it follows that \(a>f^{\inf}\).
  By the limit assumption, there exists \(R>0\) such that
  \(\|x\|>R\Rightarrow f(x)<a\).
  Thus \(\{x:f(x)\ge a\}\subseteq\{x:\|x\|\le R\}\), which is bounded.
  Hence \(f\) has bounded superlevel sets.

  \paragraph{\(\text{\ref{item:liminf}} \Rightarrow
  \text{\ref{item:radially_unbounded}}\).} If \(f^{\inf}=-\infty\),
  then \(\lim_{\|x\|\to\infty} f(x)=-\infty\), so \(-f(x)\to
  +\infty\) and \(-f\) is radially unbounded.

  If \(f^{\inf}\in\R\), then \(f(x)>f^{\inf}\) for all \(x\) and
  \(f(x)\to f^{\inf}\) as \(\|x\|\to\infty\).
  Hence \(f(x)-f^{\inf}\to 0^+\) and so
  \[
    -\log\big(f(x)-f^{\inf}\big)\longrightarrow +\infty\quad\text{as
    }\|x\|\to\infty,
  \]
  i.e., \(-\log(f(x)-f^{\inf})\) is radially unbounded.

  \paragraph{\(\text{\ref{item:radially_unbounded}} \Rightarrow
  \text{\ref{item:liminf}}\).} If \(-f(x)\) is radially unbounded,
  then \(f(x)\to-\infty\) as \(\|x\|\to\infty\), so
  \(f^{\inf}=-\infty\) and \(f(x)>f^{\inf}\) holds trivially.

  Otherwise, assume there exists a finite \(f^{\inf}=\inf
  f(\R^n)\in\R\) such that \(-\log(f(x)-f^{\inf})\) is radially unbounded.
  For \(-\log(f(x)-f^{\inf})\) to be well-defined everywhere, it must
  hold that \(f(x)-f^{\inf}>0\) for all \(x\), i.e., \(f(x)>f^{\inf}\).
  Moreover, \(-\log(f(x)-f^{\inf})\to+\infty\) implies
  \(f(x)-f^{\inf}\to 0^+\) as \(\|x\|\to\infty\), hence \(f(x)\to f^{\inf}\).
  This is precisely statement~\ref{item:liminf}.

\end{proof}

\subsection{Proof of Theorem~\ref{thm:uniform_goal_reaching}}
\label{app:proof_uniform_goal_reaching}
\begin{proof}
  The proof begins by introducing the definitions
  \eqref{eq:vmin}--\eqref{eq:R} for the quantities $\tau(d^{\circ})$,
  $\tau^{\mathrm{b}}(d^{\circ},d^{*})$, and $\delta(d^{\circ})$.
  The proof then explains, step by step, how these definitions
  guarantee Claims \ref{claim:uniform_overshoot_bound},
  \ref{claim:uniform_reaching_time}, and
  \ref{claim:distribution_reaching_time}, making clear how each
  object is constructed from scratch.
  Specifically, the following quantities are defined:
  \begin{align}
    v^{\min}(d^{\circ})
    & := \min\{\kappalow(\state) : \goaldist(\state)\le d^{\circ}\}
    \label{eq:vmin}
    \\
    \mathbb{V}(d^{\circ})
    & := \{\state \in \states: \kappalow(\state) \ge
    v^{\min}(d^{\circ})\}
    \label{eq:superset}
    \\
    v^{\max}(d^{\circ})
    & := \max\{\kappahigh(\state) : \state \in \mathbb{V}(d^{\circ})\}
    \label{eq:vmax}
    \\
    \tau(d^{\circ})
    & := 1 + \left\lfloor\dfrac{v^{\max}(d^{\circ}) -
    v^{\min}(d^{\circ})}{\nu}\right\rfloor
    \label{eq:T}
    \\
    d^{\bar{\mathrm{p}}}(d^{\circ})
    & := \sup\{\bar{\transit}(s, a)
    : \state \in \mathbb{V}(d^{\circ}),\,\action\in\actions\}
    \label{eq:Rrho}
    \\
    d^{\max}(d^{\circ})
    & := \max\left(d^{\circ},\, d^{\bar{\mathrm{p}}}(d^{\circ})\right)
    \label{eq:R}
    \\
    \delta(d^{\circ})
    & := \beta(d^{\max}(d^{\circ}), 0).
    \label{eq:epsH}
  \end{align}
  In the equations above, $v^{\min}(d^{\circ})$ and
  $v^{\max}(d^{\circ})$ bound critic values, the superlevel set
  $\mathbb{V}(d^{\circ})$ defines the operational domain of the
  baseline policy, and $\tau(d^{\circ})$ bounds the number of
  critic-improvement iterations, while
  $d^{\bar{\mathrm{p}}}(d^{\circ})$ bounds the maximum one-step
  transition from any state within the superlevel set
  $\mathbb{V}(d^{\circ})$, and $d^{\max}(d^{\circ})$ provides a
  composite bound incorporating both initial distance and transition magnitudes.

  Further, $\tau^{\mathrm{b}}(d^{\circ},d^{*})$ is defined as the
  minimum time steps for the baseline policy to drive any state with
  $\goaldist(\state) \leq d^{\max}(d^{\circ})$ to within $d^{*}$ of the goal:
  \begin{equation}
    \tau^{\mathrm{b}}(d^{\circ},d^{*})
    :=
    \max\left\{
      1,\,
      \left\lceil
      -\log\left(\xi^{-1}\left(
          \tfrac{d^{*}}{\kappa(d^{\max}(d^{\circ}))}
      \right)\right)\right\rceil
    \right\},
    \label{eq:Tfallback}
  \end{equation}
  where $\kappa, \xi \in \K_{\infty}$ are functions such that
  $\beta(d,t) \leq \kappa(d)\xi(e^{-t})$ for all $d \geq 0, t \geq 0$
  (a standard decomposition of $\mathcal{KL}$ functions,
  see~\cite[Lemma 8]{Sontag1998Comments}).

  A series of results is now established that together proves the
  theorem's claims:
  \begin{lemma}\label{prop:disallowed_superset}
    Whenever $\Action_t = \Action^{\learningpolicy}_t$, it holds that
    $\State_t \in \mathbb{V}(d^{\circ})$.
  \end{lemma}
  \begin{proof}
    The \(\learningpolicy\) is chosen in two cases:

    (1) When \(Q^{\dagger}_{t} + \nu < Q^{w_t}(\State_t,
    \Action^{\learningpolicy}_t)\): At initialization,
    \(Q^{\dagger}_0\) is initialized to \(-\infty\).
    On the first iteration, it updates to \(Q^{\dagger}_1 =
    Q^{w_0}(\State_0, \Action^{\learningpolicy}_0)\).
    Moreover, \(Q^{\dagger}_t\) is nondecreasing by construction.
    Hence, for \(t \ge 1\),
    \[
      Q^{\dagger}_t \ge Q^{w_0}(\State_0,
      \Action^{\learningpolicy}_0) \ge \kappalow(\State_0) \ge
      v^{\min}(d^{\circ}),
    \]
    and therefore \(\State_t \in \mathbb{V}(d^{\circ})\) for all $t \ge 1$.
    By Assumption~\ref{ass:valuebase}, the critic satisfies
    \(\kappalow(\State_0) \le Q^{w_0}(\State_0,
    \Action^{\learningpolicy}_0)\), and since \(\goaldist(\State_0)
    \le d^{\circ}\), it follows that \(\kappalow(\State_0) \ge
    v^{\min}(d^{\circ})\), implying \(\State_0 \in \mathbb{V}(d^{\circ})\).

    (2) When \(U_t \leq \lambda^{t}\rho^{\mathrm{rel}}_t\): By
    definition, \(\rho^{\mathrm{rel}}_t = 0\) whenever
    \(Q^{w_t}(\State_t, \Action^{\learningpolicy}_t) <
    Q^{w_0}(\State_0, \Action^{\learningpolicy}_0)\), which can occur
    only for \(t \ge 1\).
    In such a case, the learning-policy branch is not executed, since
    the condition \(Q^{\dagger}_{t} + \nu < Q^{w_t}(\State_t,
    \Action^{\learningpolicy}_t)\) cannot hold, because for all \(t \ge 1\),
    \[
      Q^{\dagger}_t \ge Q^{w_0}(\State_0,
      \Action^{\learningpolicy}_0) \ge v^{\min}(d^{\circ}).
    \]
    When \(\rho^{\mathrm{rel}}_t > 0\), it follows that
    \(Q^{w_t}(\State_t, \Action^{\learningpolicy}_t) \ge
      Q^{w_0}(\State_0, \Action^{\learningpolicy}_0) \ge
    v^{\min}(d^{\circ})\), and therefore \(\State_t \in \mathbb{V}(d^{\circ})\).

  \end{proof}

  \begin{lemma}\label{prop:T}
    The quantity $\tau(d^{\circ})$ in~\eqref{eq:T} bounds the number
    of times the critic value can significantly improve:
    \[
      N_{\Value^{\dagger}} := \sum_{t=0}^{\infty}
      \mathbb{I}\{Q^{\dagger}_t + \nu < Q^{w_t}(\State_t,
      \Action^{\learningpolicy}_t)\} \leq \tau(d^{\circ})
    \]
  \end{lemma}
  \begin{proof}
    The quantity $\tau(d^{\circ})$ is well-defined because
    $v^{\min}(d^{\circ})$ and $v^{\max}(d^{\circ})$ are finite.
    This follows from the compactness of the relevant superlevel set
    $\mathbb{V}(d^{\circ})$ and the continuity of $\kappalow(\state)$
    and $\kappahigh(\state)$ (Assumption~\ref{ass:valuebase}).

    At $t = 0$, $Q^{\dagger}_0$ is initialized to $-\infty$, so the
    update condition $Q^{\dagger}_t + \nu < Q^{w_t}(\State_t,
    \Action^{\learningpolicy}_t)$ is satisfied trivially, resulting
    in an unbounded initial jump to $Q^{\dagger}_1 =
    Q^{w_0}(\State_0, \Action^{\learningpolicy}_0)$.
    For all subsequent iterations ($t \ge 1$), each occurrence of
    this condition increases $Q^{\dagger}_t$ by a finite jump of at least $\nu$.

    By Lemma~\ref{prop:disallowed_superset}, such updates can occur
    only when $\State_t \in \mathbb{V}(d^{\circ})$, where
    $Q^{w_t}(\State_t, \Action^{\learningpolicy}_t) \le v^{\max}(d^{\circ})$.
    Since $Q^{\dagger}_1 = Q^{w_0}(\State_0,
    \Action^{\learningpolicy}_0) \ge v^{\min}(d^{\circ})$, the number
    of finite jumps that can occur is bounded by $\lfloor
    (v^{\max}(d^{\circ}) - v^{\min}(d^{\circ}))/\nu \rfloor$.
    Including the initial unbounded update, the total number of
    possible improvements is therefore
    \[
      1 + \left\lfloor \frac{v^{\max}(d^{\circ}) -
      v^{\min}(d^{\circ})}{\nu} \right\rfloor = \tau(d^{\circ}),
    \]
    which completes the proof.
  \end{proof}
  \begin{lemma}\label{prop:Tq}
    Let \( T^{\mathrm{rel}} := \inf\{t \ge 0\!: U_{k} \ge \lambda^k
      \relprob \forall k \ge t\}.
    \)

    \text{\ref{item:bounded_superlevel_sets}} $T^{\mathrm{rel}}$
    bounds the total number of random acceptances:
    $\sum_{t=0}^{\infty} \mathbb{I}\{U_t < \lambda^t \relprob\} \leq
    T^{\mathrm{rel}}$.

    \text{\ref{item:liminf}} $T^{\mathrm{rel}}$ is almost surely finite.

    \text{\ref{item:radially_unbounded}} For all $t \in
    \mathbb{Z}_{\ge 0}$, $\PP{T^{\mathrm{rel}} \leq t} =
    \prod_{k=t}^{\infty}(1 - \lambda^k \relprob)$, and
    $\prod_{k=t}^{\infty}(1 - \lambda^k \relprob) \to 1$ as $t \to \infty$.
  \end{lemma}
  \begin{proof}
    \text{\ref{item:bounded_superlevel_sets}} Since
    $\rho^{\mathrm{rel}}_t \leq \relprob$ for all $t$:
    \begin{equation*}
      \sum_{t=0}^{\infty} \mathbb{I}\{U_t < \lambda^t
      \rho^{\mathrm{rel}}_t\} \leq \sum_{t=0}^{\infty}
      \mathbb{I}\{U_t < \lambda^t \relprob\} =: N_{\relprob}
    \end{equation*}

    Furthermore, $N_{\relprob} \leq \sum_{t=0}^{T^{\mathrm{rel}}-1} 1
    = T^{\mathrm{rel}}$, since by definition of $T^{\mathrm{rel}}$,
    the event $U_t < \lambda^t \relprob$ cannot occur for $t \geq
    T^{\mathrm{rel}}$.

    \text{\ref{item:liminf}} The Borel-Cantelli
    lemma~\cite{billingsley1995probability} ensures that
    $N_{\relprob}$ is almost surely finite because
    $\sum_{t=0}^{\infty} \lambda^t \relprob < \infty$, which implies
    that $T^{\mathrm{rel}}$ is almost surely finite as well.

    \text{\ref{item:radially_unbounded}} The event
    $\{T^{\mathrm{rel}} \leq t\}$ occurs if and only if $U_k \geq
    \lambda^t \relprob$ for all $k \geq t$.
    Since the $U_k$ are independent, it follows that:
    \[
      \textstyle \PP{T^{\mathrm{rel}} \leq t} = \prod_{k=t}^{\infty}
      \PP{U_k \geq \lambda^t \relprob} = \prod_{k=t}^{\infty} (1 -
    \lambda^k \relprob)\]

    Finally, $\prod_{k=t}^{\infty} (1 - \lambda^k \relprob) \to 1$ as
    $t \to \infty$ because $\sum_{k=t}^{\infty} \log(1 - \lambda^k
    \relprob) \to 0$ as $t \to \infty$, which follows from
    $\sum_{t=0}^{\infty} \lambda^t \relprob < \infty$.
  \end{proof}
  \begin{lemma}\label{prop:next_state_invokefallback}
    If $\State_t \in \mathbb{V}(d^{\circ})$ and $\Action_t =
    \Action^{\learningpolicy}_t$, then:

    \text{\ref{item:bounded_superlevel_sets}} The next state
    satisfies $\|\State_{t+1}\| \leq d^{\bar{\mathrm{p}}}(d^{\circ})$
    almost surely.

    \text{\ref{item:liminf}} Both $\state_0$ and all states with
    $\|\state\| \leq d^{\bar{\mathrm{p}}}(d^{\circ})$ satisfy
    $\goaldist(\state) \leq d^{\max}(d^{\circ})$.
  \end{lemma}
  \begin{proof}
    \text{\ref{item:bounded_superlevel_sets}} By the definition of
    $\bar{p}$ (from the system assumption) and
    $d^{\bar{\mathrm{p}}}(d^{\circ})$ in~\eqref{eq:Rrho}, when
    $\State_t \in \mathbb{V}(d^{\circ})$ and $\Action_t =
    \Action^{\learningpolicy}_t$, it holds that $\|\State_{t+1}\|
    \leq \bar{p}(\State_t, \Action^{\learningpolicy}_t) \leq
    d^{\bar{\mathrm{p}}}(d^{\circ})$ almost surely.

    \text{\ref{item:liminf}} By definition, $\goaldist(\state_0) \leq
    d^{\circ} \leq d^{\max}(d^{\circ})$.
    For any state with $\|\state\| \leq
    d^{\bar{\mathrm{p}}}(d^{\circ})$, it follows that
    $\goaldist(\state) \leq \|\state\| \leq
    d^{\bar{\mathrm{p}}}(d^{\circ}) \leq d^{\max}(d^{\circ})$.
  \end{proof}

  \begin{lemma}\label{prop:Tfallback}
    For any state $\state_0$ with $\goaldist(\state_0) \leq
    d^{\max}(d^{\circ})$:

    \text{\ref{item:bounded_superlevel_sets}}
    $\PP{\goaldist(\State_t^{\baselinepolicy}(\state_0)) \leq d^{*}
    \text{ for all } t \geq \tau^{\mathrm{b}}(d^{\circ},d^{*})} \geq 1-\eps$.

    \text{\ref{item:liminf}}
    $\PP{\goaldist(\State_t^{\baselinepolicy}(\state_0)) \leq
    \delta(d^{\circ}) \text{ for all } t \geq 0} \geq 1-\eps$.
  \end{lemma}
  \begin{proof}
    By Assumption~\ref{ass:uniform_fallback},
    \[
      \PP{\goaldist(\State_t^{\baselinepolicy}(\state_0)) \leq
      \beta(\goaldist(\state_0),t) \text{ for all } t} \geq 1-\eps
    \]

    \text{\ref{item:bounded_superlevel_sets}} From the definition of
    $\tau^{\mathrm{b}}(d^{\circ},d^{*})$ in \eqref{eq:Tfallback}, it
    follows that $\beta(\goaldist(\state_0),t) \leq d^{*}$ for all $t
    \geq \tau^{\mathrm{b}}(d^{\circ},d^{*})$ when
    $\goaldist(\state_0) \leq d^{\max}(d^{\circ})$.

    \text{\ref{item:liminf}} For any $t \geq 0$, it holds that
    $\beta(\goaldist(\state_0),t) \leq \beta(d^{\max}(d^{\circ}),0) =
    \delta(d^{\circ})$ when $\goaldist(\state_0) \leq d^{\max}(d^{\circ})$.
  \end{proof}
  \textit{Conclusion of the proof of Theorem \ref{thm:uniform_goal_reaching}}.
  From Lemmas~\ref{prop:T} and~\ref{prop:Tq}, the total number of
  times the baseline policy is chosen is at most $\tau(d^{\circ}) +
  T^{\mathrm{rel}}$ almost surely.

  From Lemmas~\ref{prop:disallowed_superset}
  and~\ref{prop:next_state_invokefallback}, whenever the algorithm
  switches to the baseline policy, the state satisfies
  $\goaldist(\state) \leq d^{\max}(d^{\circ})$.

  From Lemma~\ref{prop:Tfallback}, after running the baseline policy
  for $\tau^{\mathrm{b}}(d^{\circ},d^{*})$ steps from any such state,
  the system stays within $d^{*}$ of the goal thereafter with
  probability at least $1-\eps$.

  Therefore, define the reaching time as
  \begin{equation}
    \label{eq:Tcalfwpolicy}
    T(d^{\circ},d^{*}) := (\tau(d^{\circ}) +
    T^{\mathrm{rel}})\tau^{\mathrm{b}}(d^{\circ},d^{*})
  \end{equation}
  All three claims now follow:

  Claim~\ref{claim:uniform_overshoot_bound} follows from
  Lemma~\ref{prop:Tfallback}\text{\ref{item:liminf}}, showing that
  the maximum deviation from the goal is bounded by $\delta(d^{\circ})$.

  Claim~\ref{claim:uniform_reaching_time} follows as the baseline
  policy is used at most $\tau(d^{\circ}) + T^{\mathrm{rel}}$ times,
  and $\tau^{\mathrm{b}}(d^{\circ},d^{*})$ baseline steps suffice to
  maintain the system within $d^*$ of the goal with probability
  $1-\eps$, whether starting from the initial state or after baseline
  policy use.

  Claim~\ref{claim:distribution_reaching_time} follows directly from
  Lemma~\ref{prop:Tq}\text{\ref{item:radially_unbounded}} and the
  definition of $T(d^{\circ},d^{*})$ in~\eqref{eq:Tcalfwpolicy}.

  The proof of Theorem~\ref{thm:uniform_goal_reaching} is complete.
\end{proof}

\subsection{Proof of Corollary~\ref{thm:uniform_goal_reaching_value}}
\label{app:proof_uniform_value_corollary}
\begin{proof}
  Take the proof of \Cref{thm:uniform_goal_reaching} and perform the
  following replacements everywhere: $Q^{w}(\state,\action)\to
  \Value^{w}(\state)$, $Q^{\dagger}_{t}\to \Value^{\dagger}_{t}$, and
  $\rho^{\mathrm{rel}}_t\to
  \relprob\mathbb{I}\{\Value^{w_t}(\State_t)\ge
  \Value^{w_0}(\State_0)\}$, while relabeling Assumptions
  \ref{ass:valuebase} and \ref{ass:uniform_fallback} to
  \ref{ass:valuebase_value} and \ref{ass:uniform_fallback_value};
  with these substitutions, every definition, lemma, and bound is
  unchanged, and the corollary follows.
\end{proof}

\subsection{Proof of Theorem~\ref{thm:transfer}}
\label{app:proof_transfer}
\begin{proof}
  The state processes $\State_{0:T}^{\learningpolicy}$ and
  $\State_{0:T}^{\policy_t}$ are sampled independently, as in
  Definition~\ref{dfn:trajectory-distance}.

  \textit{Step~1 (geometric argument).} Define the \emph{tube event}
  and its complement:
  \[
    E_\delta := \left\{\sup_{0 \le t < T}
      \|\State_t^{\learningpolicy} - \State_t^{\policy_t}\| \le
    \delta\right\}, \; E_\delta^c := \left\{\sup_{0 \le t < T}
    \|\State_t^{\learningpolicy} - \State_t^{\policy_t}\| > \delta\right\}.
  \]
  Claim: $\left\{\tau_T(d^*,\, \policy_t) < T\right\} \cap E_\delta
  \subseteq \left\{\tau_T(d^* + \delta,\, \learningpolicy) < T\right\}$.

  Indeed, suppose $\tau_T(d^*,\, \policy_t) = t^* < T$.
  Then for every $t \in \{t^*,\ldots,T{-}1\}$,
  $\goaldist\left(\State_t^{\policy_t}\right) \le d^*$.
  On $E_\delta$, for every such $t$:
  \[
    \goaldist\left(\State_t^{\learningpolicy}\right) \;\le\;
    \|\State_t^{\learningpolicy} - \State_t^{\policy_t}\| +
    \goaldist\left(\State_t^{\policy_t}\right) \;\le\; \delta + d^*,
  \]
  so $\tau_T(d^* + \delta,\, \learningpolicy) \le t^* < T$.

  \textit{Step~2 (probability estimate).} Using the inclusion from Step~1:
  \[
    \PP{\tau_T(d^* + \delta,\, \learningpolicy) < T} \ge
    \PP{\left\{\tau_T(d^*,\, \policy_t) < T\right\} \cap E_\delta}.
  \]
  Decomposing by $E_\delta$:
  \begin{multline}
    \PP{\left\{\tau_T(d^*,\, \policy_t) < T\right\} \cap E_\delta}
    = \\
    \PP{\tau_T(d^*,\, \policy_t) < T}
    - \PP{\left\{\tau_T(d^*,\, \policy_t) < T\right\} \cap E_\delta^c}
    \\
    \ge \PP{\tau_T(d^*,\, \policy_t) < T} -
    \PP{E_\delta^c}.
    \label{eq:decompose}
  \end{multline}

  \textit{Step~3 (Markov bound on tube exit).} By Markov's inequality:
  \begin{multline}\label{eq:markov_tube}
    \PP{E_\delta^c}
    = \PP{\sup_{0 \le t < T}
      \|\State_t^{\learningpolicy} - \State_t^{\policy_t}\|
    > \delta}
    \\
    \le \frac{\E{\sup_{0 \le t < T}
        \|\State_t^{\learningpolicy} -
    \State_t^{\policy_t}\|}}{\delta}
    = \frac{\DT(\learningpolicy, \policy_t)}{\delta}
    \le \frac{\Delta_T}{\delta}\,.
  \end{multline}

  \textit{Step~4 (combine).} Substituting
  Assumption~\ref{ass:transfer_settling} and \eqref{eq:markov_tube}
  into~\eqref{eq:decompose}:
  \begin{multline}
    \PP{\tau_T(d^* + \delta,\, \learningpolicy) < T} \;\ge\;
    \PP{\tau_T(d^*,\, \policy_t) < T} - \frac{\Delta_T}{\delta}
    \\
    \;\ge\; (1 - \eps) - \frac{\Delta_T}{\delta}\,.
    \qedhere
  \end{multline}
\end{proof}

\section*{CRediT authorship contribution statement}

Anton Bolychev and Georgiy Malaniya contributed equally to this work.

Anton Bolychev: Formal analysis, Investigation, Software, Validation,
Visualization, Writing -- original draft.

Georgiy Malaniya: Methodology, Formal analysis, Software, Validation,
Writing -- original draft.

Sinan Ibrahim: Software, Visualization, Data curation.

Pavel Osinenko: Conceptualization, Methodology, Supervision, Writing
-- review and editing.

\section*{Declaration of competing interest}

The authors declare that they have no known competing financial
interests or personal relationships that could have appeared to
influence the work reported in this paper.

\section*{Acknowledgements}

Research reported in this publication was financially supported by
the Russian Science Foundation (RSF) grant No.~25-21-00872.

\section*{Data availability}

All simulation code, source code required to reproduce the
experimental runs, the resulting run data, and scripts used to
generate the figures are available in the project repository: \gitrepo.

\bibliographystyle{elsarticle-num}
\bibliography{
  bib/Osinenko__Apr2025.bib, bib/AIDA__Apr2025.bib, bib/calf-enhance.bib
}

\end{document}

%% file: preamble-AIDA.tex
\newcommand{\eps}{{\varepsilon}}										
\newcommand{\nrm}[1]{\left\lVert#1\right\rVert}							
\providecommand{\norm}[1]{\nrm{#1}}										
\newcommand{\E}[2][{}]{\mathbb E_{#1}\left[#2\right]}					
\newcommand{\PP}[1]{\mathbb P\left[#1\right]}							
\newcommand{\ra}{\rightarrow}											
\newcommand{\ie}{\unskip, i.\,e.,\xspace}								

\newcommand{\N}{{\mathbb{N}}}											
\newcommand{\R}{{\mathbb{R}}}											
\newcommand{\state}{s}													
\newcommand{\State}{S}													
\newcommand{\states}{\mathbb S}											
\newcommand{\action}{a}													
\newcommand{\Action}{A}													
\newcommand{\actions}{\mathbb A}										
\newcommand{\policy}{\pi}												
\newcommand{\policies}{\Pi}												
\newcommand{\transit}{p}												
\newcommand{\reward}{r}													
\newcommand{\Value}{v}													
\newcommand{\G}{\ensuremath{\mathbb{G}}}								
\newcommand{\K}{\ensuremath{\mathcal{K}}\xspace}						
\newcommand{\KL}{\ensuremath{\mathcal{KL}}\xspace}						